\documentclass[twoside,11pt]{article}
\usepackage[abbrvbib, preprint]{jmlr2e}
\usepackage{float}
\usepackage{amsmath,amsfonts,amssymb}
\usepackage{subfigure}
\usepackage{algorithmic}
\usepackage{algorithm}
\usepackage{enumerate}
\usepackage{xcolor}
\usepackage{abbe}
\usepackage{booktabs, colortbl, array} 
\usepackage{tikz}
\usetikzlibrary{arrows.meta,shapes.geometric,positioning}
\definecolor{Gray}{gray}{0.9} 

\global\long\def\CE#1#2{\EE\left[\left.#1\right|#2\right]}%
\global\long\def\CP#1#2{\PP\left(\left.#1\right|#2\right)}%
\global\long\def\norm#1{\left\Vert #1\right\Vert }%

\global\long\def\Pcal{\mathcal{P}}%
\global\long\def\mA{\mathcal{A}}%

\global\long\def\Reward#1#2{Y_{#1,#2}}%
\global\long\def\Error#1#2{\eta_{#1,#2}}%
\global\long\def\Estimator#1{\widehat{\theta}_{#1}}%
\global\long\def\Diff#1#2{\Delta_{#1,#2}}%
\global\long\def\Action#1{a_{#1}}%
\global\long\def\Maxeigen#1{\lambda_{\max}\!\left(#1\right)}%
\global\long\def\Trace#1{\text{Tr}\left(#1\right)}%
\global\long\def\Regret#1{regret\ensuremath{(#1)}}%
\usepackage{lastpage}
\jmlrheading{24}{2025}{1-\pageref{LastPage}}{12/24}{??/??}{21-0000}{Wonyoung Kim}

\ShortHeadings{Adaptive Data Augmentation for  Thompson Sampling}{Kim}
\firstpageno{1}

\usepackage{babel}
\providecommand{\lemmaname}{Lemma}
\providecommand{\propositionname}{Proposition}

\providecommand{\theoremname}{Theorem}

\makeatother

\begin{document}

\title{Adaptive Data Augmentation for Thompson Sampling}
\author{\name Wonyoung Kim 
\email wyk7@cau.ac.kr \\  
\addr Department of Artificial Intelligence\\         
Chung-Ang University \\         
Seoul, Republic of Korea
}
\editor{My editor}
\maketitle

\begin{abstract}%
In linear contextual bandits, the objective is to select actions that maximize cumulative rewards, modeled as a linear function with unknown parameters. 
Although Thompson Sampling performs well empirically, it does not achieve optimal regret bounds. 
This paper proposes a nearly minimax optimal Thompson Sampling for linear contextual bandits by developing a novel estimator with the adaptive augmentation and coupling of the hypothetical samples that are designed for efficient parameter learning. 
The proposed estimator accurately predicts rewards for all arms without relying on assumptions for the context distribution.
Empirical results show robust performance and significant improvement over existing methods.
\end{abstract}

\begin{keywords}   
Thompson Sampling, minimax optimal regret bound, hypothetical bandit problem, adaptive data augmentation, coupling.
\end{keywords}

\section{Introduction}

Linear contextual bandits (LinCB) provide a simple yet powerful framework for sequential decision-making. 
At each decision epoch, an agent selects an action from a set of context vectors to maximize cumulative rewards, assumed to be linear functions of the chosen contexts. 
A special case is the multi-armed bandit (MAB) problem, in which the contexts are standard Euclidean basis vectors. Compared with more complex models—such as generalized linear models or deep neural networks, which demand substantial updates at each step—LinCB and MAB offer high computational efficiency. 
LinCB methods have been widely applied in fields such as e-commerce personalization \citep{hsu2020recommending}, revenue management \citep{ferreira2018online}, clinical trials \citep{murphy2005experimental}, political-science experiments \citep{offer2021adaptive}, and A/B testing in marketing \citep{satyal2018ab}, as comprehensively surveyed by \citet{bouneffouf2020survey}.

Two primary algorithmic paradigms dominate the LinCB literature: the upper-confidence-bound algorithm (\texttt{LinUCB}) and Thompson Sampling (\texttt{LinTS}). 
\texttt{LinUCB} selects the arm that maximizes an upper confidence bound on its reward, whereas \texttt{LinTS} samples a parameter from its posterior (or estimated) distribution and selects the arm with the highest sampled reward. Empirical studies \citep{chapelle2011} consistently demonstrate the superiority of \texttt{LinTS} over \texttt{LinUCB} across various scenarios. Nevertheless, a significant gap persists between the best-known frequentist regret bound for \texttt{LinTS} \citep{abeille2017linear} and the minimax lower bound \citep{lattimore2020bandit}. Closing this gap is challenging due to the selective reward observation inherent in \texttt{LinTS}, complicating variance control for optimal-arm reward estimation.

Recent LinCB research has leveraged advanced statistical methodologies to enhance algorithmic performance and theoretical guarantees. 
Techniques from high-dimensional parameter estimation \citep{buhlmann2011statistics}, optimal experimental design \citep{smith1918standard,guttorp2009karl}, and Bayesian optimization \citep{mockus2005bayesian} have been adapted to bandit settings. 
More recently, missing-data techniques have been employed to bridge the regret-bound gap by estimating missing rewards as though rewards from all arms were observed at every round \citep{kim2019doubly,kim2021doubly}. 
Unlike conventional estimators, which reduce errors solely for selected arms, these methods seek convergence across all arms. 
However, their reliance on inverse-probability weighting introduces variance that scales with the number of arms. To mitigate this, existing methods typically impose restrictive assumptions—such as independent-and-identically-distributed (IID) contexts or particular diversity conditions—limiting broader applicability. 
Addressing this limitation necessitates resolving open issues in both missing-data and LinCB literatures.

This paper addresses this critical gap by introducing a novel estimation approach capable of learning rewards for all arms without relying on IID or diversity conditions. 
In the proposed framework, a hypothetical bandit problem tailored for efficient parameter estimation is constructed. 
This hypothetical setup employs a set of orthogonal basis vectors, preserving the covariance structure of the original contexts while significantly reducing the effective number of arms. By coupling the hypothetical and original problems, the resulting estimator achieves a novel self-normalized bound based on a Gram matrix encompassing contexts from all arms, including unselected ones. 
The proposed algorithm, equipped with this new estimator, attains the minimax optimal regret bound up to logarithmic factors without restrictive assumptions on context distributions.

The remainder of the paper is organized as follows. Section~\ref{sec:related_works} reviews relevant literature on LinCB, highlighting key contributions of the proposed approach. 
Section~\ref{sec:LinCB} presents the formal problem formulation. Section~\ref{sec:proposed_method} details the proposed estimator and algorithm along with their theoretical justifications. 
Section~\ref{sec:regret_analysis} provides a rigorous regret analysis, establishing the minimax optimality of the proposed method. 
Finally, Section~\ref{sec:experiment} empirically validates the effectiveness of the proposed algorithm across various benchmark scenarios.

\section{Related Literature}
\label{sec:related_works}

\begin{table}[t]
\centering
\caption{Comparison of regret bounds and assumptions on  \texttt{LinTS}}
\label{tab:ts-regret-comparison}
\begin{tabular}{@{}lcc@{}}
\toprule
\textbf{Reference} & \textbf{Regret bound} & \textbf{Key assumptions}\\
\midrule
\citet{agrawal2013thompson}              & $\tilde{O}\bigl(d^{3/2}\sqrt{T}\bigr)$ & Standard \\[2pt]
\citet{abeille2017linear}              & $\tilde{O}\bigl(d^{3/2}\sqrt{T}\bigr)$ & Standard \\[2pt]
\citet{dimakopoulou2019balanced}       & $\tilde{O}\bigl(d^{3/2}\sqrt{T}\bigr)$ & Standard \\[2pt]
\citet{kim2021doubly}                  & $\tilde{O}\bigl(d\sqrt{T}\bigr)$ & IID contexts from special distributions \\ [2pt]
\citet{huix2023tight}                  & $\tilde{O}\bigl(d\sqrt{T}\bigr)$ & Gaussian prior on parameter \\[2pt]
\textbf{This work}                     & $\tilde{O}\bigl(d\sqrt{T}\bigr)$ & Standard \\
\bottomrule
\end{tabular}
\end{table}

The linear contextual bandit (LinCB) problem, introduced by \citet{abe1999associative}, has become foundational in sequential decision-making tasks. Two predominant algorithmic frameworks for LinCB are the upper-confidence-bound algorithm (\texttt{LinUCB}) and Thompson Sampling (\texttt{LinTS}). \texttt{LinUCB}, which selects the arm maximizing the upper confidence bound of its reward, has been extensively studied \citep{auer2002using, dani2008stochastic, rusmevichientong2010linearly, chu2011contextual, abbasi2011improved}. In contrast, \texttt{LinTS}, incorporating randomization by sampling from an estimated or posterior distribution of rewards, has attracted significant attention \citep{agrawal2013thompson, abeille2017linear}. Empirical studies, such as \citet{chapelle2011}, demonstrate that \texttt{LinTS} frequently outperforms \texttt{LinUCB} in practical scenarios, including online advertising and recommendation systems.

Theoretically, given contexts of dimension $d$ and time horizon $T$, \texttt{LinUCB} achieves a regret bound of $\tilde{O}(d\sqrt{T})$, matching the minimax lower bound of $\Omega(d\sqrt{T})$ up to logarithmic factors \citep{lattimore2020bandit}. In contrast, \texttt{LinTS} currently achieves a higher regret bound of $\tilde{O}(d^{3/2}\sqrt{T})$, and improving this bound remains an open problem. 

Table~\ref{tab:ts-regret-comparison} summarizes existing regret bounds and associated assumptions for various \texttt{LinTS} methods. Recent studies have contributed to narrowing this gap. \citet{kim2021doubly} introduced a doubly robust (DR) estimator instead of the ridge estimator, achieving a regret bound of $\tilde{O}(\alpha^{-1}\sqrt{T})$ under independent contexts with strictly positive minimum eigenvalue $\alpha > 0$. Special cases with $\alpha^{-1} = O(d)$ are further analyzed by \citet{bastani2021mostly} and \citet{kim2023double}. However, if $\alpha$ is extremely small (e.g., fixed or highly correlated contexts), this bound can be worse than previous guarantees. \citet{huix2023tight} achieved the minimax rate of $\tilde{O}(d\sqrt{T})$, but under the assumption of a Gaussian prior distribution on the parameter, yielding Bayesian rather than worst-case frequentist guarantees. For the multi-armed bandit (MAB) setting, \citet{agrawal2017near} and \citet{zhu2020thompson} obtained minimax-optimal bounds, but analogous results for LinCB with arbitrary contexts remain unresolved.

Statistical techniques, particularly those addressing missing data, have significantly advanced LinCB research. Methods such as inverse probability weighting (IPW) and doubly robust estimation (DR) \citep{bang2005doubly} tackle the selective reward observation problem by treating unselected rewards as missing data. \citet{dimakopoulou2019balanced} employed IPW in \texttt{LinTS}, obtaining a regret bound of $\tilde{O}(d^{3/2}\sqrt{T})$. \citet{kim2019doubly} adapted DR methods to sparse, high-dimensional linear bandits, leveraging information from unselected contexts. Subsequent works by \citet{kim2021doubly} and \citet{kim2023double} enhanced DR estimators under assumptions of stochastic contexts and generalized linear rewards, respectively. \citet{kim2023squeeze} further generalized DR methods to scenarios involving zero-probability arm selections. Despite these advancements, their reliance on IID or diversity assumptions limits broader applicability, leaving open the challenge of improving regret bounds for arbitrary contexts. This paper addresses this challenge by developing a novel estimator and algorithm that achieve the minimax optimal regret bound of $\tilde{O}(d\sqrt{T})$ without restrictive assumptions on the context distributions.

\section{Linear Contextual Bandit Problem}
\label{sec:LinCB}
This section presents the notation used throughout the paper and the formal definition of the linear contextual bandit (LinCB) problem.

\subsection{Notations}
For a natural number $n \in \mathbb{N}$, define $[n] := {1, 2, \ldots, n}$. For a positive semidefinite matrix $M \in \mathbb{R}^{d \times d}$ and a vector $x \in \mathbb{R}^d$, let $|x|_M := \sqrt{x^\top M x}$. For two matrices $A$ and $B$, write $A \succ B$ (respectively $A \succeq B$) if $A - B$ is positive definite (respectively positive semidefinite).

\subsection{Problem Formulation}
In LinCB, the environment defines a sequence of distributions  over $d$-dimensional context vectors for $K$ arms, constrained to the set. Deterministic contexts can also be represented by setting each  as a Dirac measure. The time horizon $T$ is finite but not known to the learner. 
At each round $t \in [T]$, the environment draws context vectors $(X_{1,t}, \ldots, X_{K,t})$ from $\mathcal{P}t$, where $X{k,t}$ denotes the context for arm $k$. 
Assume that $x_{\max}$ is known; if unknown, it can be replaced with $X_{\max,t} := \max_{s \in [t]} \max_{k \in [K]} |X_{k,s}|_2$.

Let $\mathcal{H}_t$ be the sigma-algebra generated by the observations until before selecting an action at round $t$, i.e.,  
\[
\Hcal_t = \bigcup_{\tau=1}^{t-1} \big[\{X_{i,\tau}\}_{i=1}^{K} \cup \{\Action{\tau}\} \cup \{\Reward{\Action{\tau}}{\tau}\}\big] \cup \{X_{i,t}\}_{i=1}^{K}.
\]  
Based on \(\Hcal_t\), the learner selects an arm \(a_t \in [K]\) and receives a reward \(\Reward{a_t}{t}\).  
In linear contextual bandits (LinCB), rewards are linear in the context, given by:  
\[
\Reward{a_t}{t} = X_{a_t,t}^\top \theta_{\star} + \Error{a_t}{t},
\]  
where \(\theta_{\star} \in \RR^d\) is the unknown parameter such that \(\|\theta_{\star}\|_2 \leq \theta_{\max}\) for some unknown \(\theta_{\max} > 0\), and \(\Error{a_t}{t}\) is conditionally zero-mean and \(\sigma\)-sub-Gaussian noise:  
\[
\CE{\exp(\lambda \Error{a_t}{t})}{\Hcal_t} \leq \exp\left(\frac{\lambda^2 \sigma^2}{2}\right) \quad \text{for all } \lambda \in \RR,
\]  
for some \(\sigma \geq 0\).  

To normalize the scale of regret, following standard convention (e.g., \citealp{abbasi2011improved}), assume that $|X_{k,t}^\top \theta_{\star}| \leq 1$ for all $k \in [K]$ and $t \in [T]$.
At each round \(t\), the optimal arm \(a_t^{\star}\) is defined as  
\(a_t^{\star} := \arg\max_{i \in [K]} (X_{i,t}^\top \theta_{\star})\),  
and the instantaneous regret is:  
\[
\Regret{t} := X_{a_t^{\star},t}^\top \theta_{\star} - X_{a_t,t}^\top \theta_{\star}.
\]  
The goal is to minimize the cumulative regret over \(T\) rounds: 
\[
R(T) := \sum_{t=1}^T \Regret{t}.
\]  
This general formulation aligns with the standard LinCB setting (see, e.g., \citealp{abbasi2011improved} and \citealp{lattimore2020bandit}) and encompasses specific cases studied in \citet{kim2021doubly} and \citet{kim2023squeeze}.

\section{Proposed Method}
\label{sec:proposed_method}

This section introduces the proposed estimation scheme that enables \texttt{LinTS} to achieve a nearly minimax-optimal regret bound.
Section~\ref{sec:augment} motivates the use of hypothetical sample augmentation for parameter estimation.
Section~\ref{sec:hypo_contexts} presents a construction of hypothetical contexts that efficiently support parameter learning while minimizing the number of augmented samples.
Building on the constructed contexts, Section~\ref{sec:hypothetical_bandit} defines an adaptive hypothetical bandit problem tailored for estimation.
Section~\ref{sec:coupling} describes a resampling strategy that couples the hypothetical and original bandit problems.
Finally, Section~\ref{sec:algorithm} outlines the proposed algorithm, which incorporates the novel estimator derived from this framework.

\subsection{Augmenting Hypothetical Contexts for Linear Bandits}
\label{sec:augment}

In linear contextual bandits (LinCB), the ridge estimator with $\ell_{2}$-regularization is a widely used approach for estimating the unknown parameter. This regularization can be interpreted as augmenting the dataset with artificial observations. Let $\mathbf{e}i \in \mathbb{R}^d$ denote the $i$-th Euclidean basis vector. Then, the ridge estimator at round $t$ can be expressed as
\begin{equation}
\label{eq:ridge}
\left(\sum_{s=1}^{t} X_{a_s,s} X_{a_s,s}^{\top} + \sum_{i=1}^{d} \mathbf{e}_i \mathbf{e}_i^{\top} \right)^{-1}
\left(\sum_{s=1}^{t} Y_{a_s,s} X_{a_s,s} + \sum_{i=1}^{d} 0 \cdot \mathbf{e}_i \right),
\end{equation}
which is equivalent to augmenting the dataset with dummy context–reward pairs $(\mathbf{e}_i, 0)$ for $i \in [d]$. \citet{bishop1995training} showed that the inclusion of such artificial data can improve generalization, i.e., performance on test data that are not used in training. However, augmenting with zero-valued rewards induces shrinkage toward the origin, resulting in an estimator that is not adaptive to the observed data.

This observation motivates the use of alternative augmented samples that enhance parameter learning more effectively. Several estimators in the literature can be interpreted through the lens of such augmentation. For example, \citet{perturb20akveton} proposed a method that adds multiple random perturbations to each reward observation; this is equivalent to augmenting the dataset with multiple context–reward pairs. \citet{kim2021doubly} introduced a doubly robust (DR) estimator:
\begin{equation}
\label{eq:DR_estimator}
\left(\sum_{s=1}^{t} \sum_{k=1}^{K} X_{k,s} X_{k,s}^{\top} + \sum_{i=1}^{d} \sqrt{\lambda_t} \mathbf{e}_i \left(\sqrt{\lambda_t}\mathbf{e}_i\right)^{\top} \right)^{-1}
\left(\sum_{s=1}^{t} \sum_{k=1}^{K} X_{k,s} Y_{k,s}^{\mathrm{DR}} \right),
\end{equation}
where $\lambda_t = \Omega(\sqrt{t})$ is a regularization parameter and $Y_{k,s}^{\mathrm{DR}}$ denotes a DR pseudo-reward. This estimator can be interpreted as augmenting the unselected contexts with their corresponding unbiased pseudo-rewards $(X_{k,s}, Y_{k,s}^{\mathrm{DR}})$ for all $k \in [K]$ and $s \in [t]$.

These augmented observations construct the Gram matrix that controls the self-normalized error of the estimator and influences its generalization performance. In the ridge estimator \eqref{eq:ridge}, the Gram matrix includes only the contexts corresponding to selected arms, and thus the estimator converges within the span of those vectors. In contrast, the DR estimator \eqref{eq:DR_estimator} incorporates all contexts, yielding a more well-conditioned matrix and enabling convergence in all directions spanned by the $K$ context vectors.

The well-conditioned Gram matrix plays a critical role in determining the convergence rate of the estimator in both linear regression and bandit settings. In the experimental design literature (e.g., \citealp{smith1918standard}; \citealp{guttorp2009karl}) and linear bandits (e.g., \citealp{soare2014best}; \citealp{tao2018best}), techniques such as E-optimal design aim to maximize the minimum eigenvalue of the Gram matrix to enhance estimation quality. Consequently, designing augmentations that yield well-conditioned Gram matrices is essential for accurate parameter estimation and low regret.

Although the DR estimator’s Gram matrix includes all $K$ arms, the augmentation involves $K$ augmented samples in each round, which introduces additional variance that scales linearly with $K$. To address this, \citet{kim2021doubly} assumed that the context covariance matrix has a strictly positive minimum eigenvalue, ensuring rapid increase in the minimum eigenvalue of the covariance.

Rather than augmenting with the full set $\{(X_{k,s}, Y_{k,s}^{\mathrm{DR}}) : k \in [K], s \in [t]\}$, the proposed method constructs a hypothetical dataset with reduced number of augmented samples by identifying a set of orthogonal eigenvectors that are informative for learning all $K$ rewards in each round. Moreover, instead of using dummy vectors such as $(\mathbf{e}_i, 0)$, the proposed estimator augments basis vectors orthogonal to the span of the observed contexts, thereby adaptively improving generalization to future inputs.

\subsection{Design of Hypothetical Contexts}
\label{sec:hypo_contexts}

At round \(t\), let \(a_t \sim \pi_t\) denote the arm drawn according to the policy \(\pi_t\). Before selecting \(a_t\), a set of hypothetical contexts is constructed to preserve the covariance structure of the original context vectors. Define \(G_t := \sum_{k \in [K] \setminus \{a_t\}} X_{k,t} X_{k,t}^{\top}\), and let \(r_t\) be its rank. Since \(G_t\) is real, symmetric, and positive semidefinite, it admits an eigen-decomposition \(G_t = \sum_{i=1}^{r_t} \lambda_{i,t} u_{i,t} u_{i,t}^{\top}\), where \(\lambda_{1,t}, \dots, \lambda_{r_t,t}\) are the positive eigenvalues and \(u_{1,t}, \dots, u_{r_t,t}\) are the corresponding orthonormal eigenvectors.
Define \(r_t + 1\) \emph{hypothetical contexts} as:
\[
Z_{i,t} =
\begin{cases}
\sqrt{\lambda_{i,t}}\,u_{i,t}, & \text{for } i = 1, \dots, r_t,\\
X_{a_t,t}, & \text{for } i = r_t+1.
\end{cases}
\]
This construction satisfies the identity
\begin{equation}
\sum_{i=1}^{r_t+1} Z_{i,t} Z_{i,t}^{\top}
= \sum_{k=1}^{K} X_{k,t} X_{k,t}^{\top},
\label{eq:Gram_equiv}
\end{equation}
ensuring that the compressed set \(\{Z_{i,t}\}_{i=1}^{r_t+1}\) exactly recovers the Gram matrix of the original \(K\) contexts, while reducing the number of arms to \(r_t + 1\).

To replace the artificial augmentation \(\{(\mathbf{e}_i, 0) : i \in [d]\}\) in the ridge estimator~\eqref{eq:ridge}, an orthogonal basis is constructed at selected rounds. Given hyperparameters \(\delta \in (0,1)\) and \(\gamma \in (0,1)\), define:
\begin{equation}
h_t := \left\lceil \frac{2}{\frac{1}{2}-e^{-1}} \frac{d}{1-\gamma} \log \frac{d(t+1)^2}{\delta} \right\rceil,
\label{eq:h_t}
\end{equation}
which specifies the number of rounds allocated for orthogonal basis augmentation.
The subset of rounds for the augmentation \(\mathcal{A}_t \subseteq [t]\) is defined recursively as:
\begin{equation}
\mathcal{A}_0 = \emptyset, \quad
\mathcal{A}_t =
\begin{cases}
\mathcal{A}_{t-1} \cup \{t\}, & \text{if } |\mathcal{A}_{t-1}| < h_t,\\
\mathcal{A}_{t-1}, & \text{otherwise}.
\end{cases}
\label{eq:A}
\end{equation}
Also define:
\begin{equation}
T_1 := \inf\{t \ge 1 : t \ge h_t\} \le \frac{8}{\frac{1}{2}-e^{-1}} \frac{d}{1 - \gamma} \left(1 + \log \frac{4}{e/2 - 1} \frac{d}{1 - \gamma} \sqrt{\frac{d}{\delta}}\right),
\label{eq:T_1}
\end{equation}
where the inequality holds by Lemma C.6 in \citet{kim2023learning}.
Thus, for all \(t \ge T_1\), it holds that \(h_t \le |\mathcal{A}_t| \le h_t + 1\)

For each \(s \in \mathcal{A}_t\), let \(r_s\) and \(\{u_{i,s}\}_{i \in [r_s]}\) denote the rank and eigenvectors of \(G_s := \sum_{k \in [K] \setminus \{a_s\}} X_{k,s} X_{k,s}^\top\). If \(r_s < d\), the Gram–Schmidt process is used to construct an orthonormal set \(\{u_{i,s}\}_{i = r_s + 1}^{d}\), orthogonal to the initial eigenvectors. The hypothetical contexts for round \(s \in \mathcal{A}_t\) are then defined as:
\[
Z_{i,s} :=
\begin{cases}
\max\{x_{\max}, 1\}\,u_{i,s}, & \text{for } i = 1, \dots, d, \\
X_{a_s,s}, & \text{for } i = d+1.
\end{cases}
\]

Now, define the number of arms in the hypothetical bandit problem at round \(s\) as:
\begin{equation}
N_s := 
\begin{cases}
r_s + 1 & \text{if } s \in [t] \setminus \mathcal{A}_t, \\
d + 1 & \text{if } s \in \mathcal{A}_t.
\end{cases}
\label{eq:N}
\end{equation}
Then, for all \(s \in [t]\) and \(i \in [N_s - 1]\), the hypothetical contexts are
\begin{equation}
Z_{i,s} := \begin{cases}
\sqrt{\lambda_{i,s}}\,u_{i,s}, & s \in [t] \setminus \mathcal{A}_t,\\
\max\{x_{\max}, 1\}\,u_{i,s}, & s \in \mathcal{A}_t,
\end{cases}
\quad\text{and}\quad
Z_{N_s,s} := X_{a_s,s}.
\label{eq:new_contexts}
\end{equation}
At round \(t\), the Gram matrix of all hypothetical contexts is given by:
\[
V_t := \sum_{s=1}^{t} \sum_{i=1}^{N_s} Z_{i,s} Z_{i,s}^\top
= \sum_{s \in [t] \setminus \mathcal{A}_t} \sum_{i=1}^{r_s + 1} Z_{i,s} Z_{i,s}^\top + \sum_{s \in \mathcal{A}_t} \sum_{i=1}^{d+1} Z_{i,s} Z_{i,s}^\top,
\]
and satisfies the following bounds.

\begin{lemma}[Gram matrix with hypothetical contexts]
\label{lem:Gram}
For all \(t \ge T_1\), the Gram matrix \(V_t\) satisfies
\begin{align*}
V_t &\succeq \sum_{s \in [t] \setminus \mathcal{A}_t} \sum_{k=1}^{K} X_{k,s} X_{k,s}^\top + \max\{x_{\max}^2, 1\} h_t I_d,
\\
V_t &\preceq \sum_{s \in [t] \setminus \mathcal{A}_t} \sum_{k=1}^{K} X_{k,s} X_{k,s}^\top + 2\max\{x_{\max}^2, 1\} h_t I_d.
\end{align*}
\end{lemma}

\begin{proof}
From~\eqref{eq:Gram_equiv}, we have:
\[
V_t = \sum_{s \in [t] \setminus \mathcal{A}_t} \sum_{k=1}^{K} X_{k,s} X_{k,s}^\top + \sum_{s \in \mathcal{A}_t} \sum_{i=1}^{d+1} Z_{i,s} Z_{i,s}^\top.
\]
For \(t \ge T_1\), we have \(h_t \le |\mathcal{A}_t| \le h_t + 1\). For each \(s \in \mathcal{A}_t\),
\[
\sum_{i=1}^{d+1} Z_{i,s} Z_{i,s}^\top = X_{a_s,s} X_{a_s,s}^\top + \max\{x_{\max}^2, 1\} \sum_{i=1}^d u_{i,s} u_{i,s}^\top.
\]
Because $\{u_{i,s}:i\in[d]\}$ are $d$ orthonormal vectors in $\RR^d$, we obtain $\sum_{i=1}^{d} u_{i,s}u_{i,s}^\top=I_d$.
Since \(X_{a_s,s} X_{a_s,s}^\top \preceq \max\{x_{\max}^2,1\} I_d\), the bounds follow.
\end{proof}

Lemma~\ref{lem:Gram} ensures that the hypothetical contexts preserve the Gram matrix of all \(K\) context vectors while reducing the number of augmented samples to \(N_s\) in each round. The proposed method retains statistical efficiency comparable to full augmentation \citep{kim2021doubly}, with less number of augmented context samples.

\subsection{A Hypothetical Linear Contextual Bandit}
\label{sec:hypothetical_bandit}

Based on the sample $a_t \sim \pi_t$, construct the hypothetical contexts $Z_{i,s}$ as previously described. For each $s \in [t]$ and $i \in [N_s]$, define the corresponding hypothetical rewards:
\[
W_{i,s} := Z_{i,s}^\top \theta_{\star} + \eta_{a_s,s},
\]
where $\eta_{a_s,s}$ is shared with the original bandit problem.

Let $\tilde{a}_s \in [N_s]$ be a hypothetical action sampled from the distribution:
\begin{equation}
\mathbb{P}(\tilde{a}_s = i) := \phi_{i,s} =
\begin{cases}
\frac{1 - \gamma}{N_s - 1}, & \text{if } i \in [N_s - 1], \\
\gamma, & \text{if } i = N_s,
\end{cases}
\label{eq:pseudo_prob}
\end{equation}
where $\gamma \in (0,1)$ determines the probability mass assigned to the original context $X_{a_s,s}$. As $\gamma$ increases, the sampling concentrates on arm $N_s$, increasing the variance of inverse-probability weights for the other arms. To mitigate this, the number of rounds with orthogonal basis augmentation, $|\mathcal{A}_t| \ge h_t$, must be sufficiently large. This construction defines a hypothetical linear bandit problem $\{(Z_{i,s}, W_{i,s}) : i \in [N_s], s \in [t]\}$ that shares the parameter $\theta_{\star}$ with the original problem.

To perform estimation, construct the following ridge estimator as a reference:
\begin{equation}
\check{\theta}_{t} := \left( \sum_{s=1}^{t} X_{a_s,s} X_{a_s,s}^\top + \gamma I_d \right)^{-1} \left( \sum_{s=1}^{t} X_{a_s,s} Y_{a_s,s} \right),
\label{eq:impute}
\end{equation}
with regularization parameter $\gamma \in (0,1)$ as used in~\eqref{eq:pseudo_prob}. Using $\check{\theta}_t$, define the pseudo-rewards:
\begin{equation}
\tilde{W}^{H(\check{\theta}_t)}_{i,s} := \left( 1 - \frac{\mathbb{I}(\tilde{a}_s = i)}{\phi_{i,s}} \right) Z_{i,s}^{\top} \check{\theta}_t + \frac{\mathbb{I}(\tilde{a}_s = i)}{\phi_{i,s}} W_{i,s}.
\label{eq:HDRY}
\end{equation}
This pseudo-reward is unbiased: $\mathbb{E}[\tilde{W}^{H(\check{\theta}_t)}_{i,s}] = Z_{i,s}^\top \theta_{\star}$ for all $i \in [N_s]$. The estimator for $\theta_{\star}$ is then:
\begin{equation}
\tilde{\theta}_t^{H(\check{\theta}_t)} := \left( \sum_{s=1}^{t} \sum_{i=1}^{N_s} Z_{i,s} Z_{i,s}^\top \right)^{-1} \left( \sum_{s=1}^{t} \sum_{i=1}^{N_s} \tilde{W}_{i,s}^{H(\check{\theta}_t)} Z_{i,s} \right).
\label{eq:Hypo_DR}
\end{equation}

However, this estimator cannot be computed directly, as the rewards $W_{i,s}$ for $i \in [N_s - 1]$ are unobserved. Only $W_{N_s,s} = Y_{a_s,s}$ is available. Thus, \eqref{eq:HDRY} is computable for all $i \in [N_s]$ only when $W_{\tilde{a}_s,s} = Y_{a_s,s}$, i.e., when the sampled context in the hypothetical bandit matches the observed context from the original bandit.
Since both models share the same parameter $\theta_{\star}$ and noise $\eta_{a_s,s}$, this condition is equivalent to $Z_{\tilde{a}_s,s} = X_{a_s,s}$. This matching event provides the motivation for the coupling technique introduced in the next section.

\subsection{Coupling the Hypothetical and Original Linear Contextual Bandits}
\label{sec:coupling}

\begin{figure}
\centering
\begin{tikzpicture}[
  >=Stealth,
  node distance = 0.7cm and 0.9cm,
  every node/.style = {font=\small},
  process/.style = {draw, rounded corners, minimum width=30mm, minimum height=11mm, align=center},
  decision/.style = {draw, diamond, aspect=2, inner sep=1pt, align=center}
]

\node[process] (theta)   {Estimate $\widehat{\theta}_{t-1}$};
\node[process, right=of theta]  (observe) {Observe $\{\,X_{k,t}:k\in[K]\}$};
\node[process, right=of observe]  (pi) {Compute the policy $\pi_t$};
\node[process, below=of pi] (sampleA) {Sample a candidate action \\ $a_t \sim \pi_t$};
\node[process, below=of sampleA] (compZ)  {Compute hypothetical contexts \\ $\{Z_{i,t}: i\in[N_t]\}$};
\node[process, below=of compZ]   (sampleC){Sample hypothetical action \\ $\tilde{a}_t\sim\tilde{\pi}_t$};

\node[decision, left=of sampleC] (decide)
      {Check \\ $X_{a_t,t}=Z_{\tilde{a}_t,t}$?};

\node[process, left=of decide] (play) 
      {Play arm $a_t$ };
\node[process, above=of play] (obsr) 
      { Observe reward \\ $Y_{a_t,t}=W_{a_t,t}$ };
\node[process, above=of obsr] (inct) 
      { $t\gets t+1$ };
\draw[->] (theta)   -- (observe);
\draw[->] (observe) -- (pi);
\draw[->] (pi) -- (sampleA);
\draw[->] (sampleA) -- (compZ);
\draw[->] (compZ)   -- (sampleC);
\draw[->] (sampleC) -- (decide);
\draw[->] (decide)  -- node[above]{yes} (play);
\draw[->] (decide)  |- node[above]{no} (sampleA);  
\draw[->] (play)    -- (obsr);
\draw[->] (obsr)    -- (inct);
\draw[->] (inct)    -- (theta);
\end{tikzpicture}
\caption{Flow diagram of the proposed coupling and resampling scheme}
\label{fig:process}
\vspace{5pt}
\end{figure}
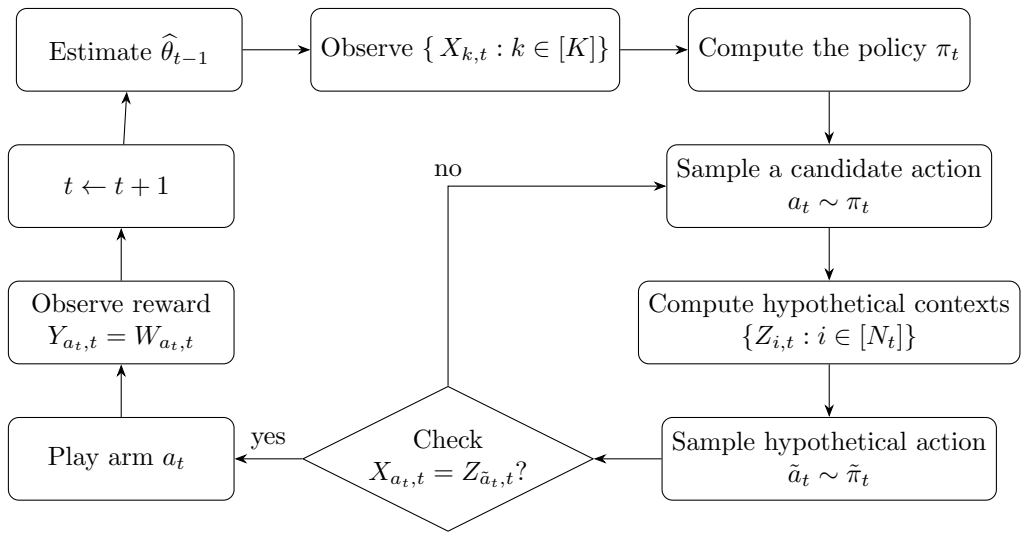

This section introduces a probabilistic method to couple the hypothetical bandit problem with the original contextual bandit. Figure~\ref{fig:process} illustrates the overall coupling and resampling process. The key observation is that the hypothetical pseudo-reward in~\eqref{eq:HDRY} is computable under the event \( \{Z_{\tilde{a}_s,s} = X_{a_s,s}\} \), which is implied by \( \{\tilde{a}_s = N_s\} \). To ensure this condition, we resample both \( a_s \sim \pi_s \) and \( \tilde{a}_s \sim \tilde{\pi}_s \) using the distribution in~\eqref{eq:pseudo_prob}.

Each resampling iteration generates updated hypothetical contexts \( \{Z_{i,s} : i \in [N_s]\} \), effectively randomizing the hypothetical contexts and rewards until the pseudo-reward becomes computable. Although rewards are collected from the original bandit, the estimation of the parameter \( \theta_{\star} \) is performed using compressed and augmented samples from the hypothetical bandit problem.

Let \( \tilde{a}_s(m) \) and \( a_s(m) \) denote the actions sampled during the \( m \)-th resampling trial in the hypothetical and original bandits, respectively. These actions are IID across trials \( m \) given $\Hcal_t$. Define the stopping time for successful coupling as \( \xi_s := \inf\{m \geq 1: X_{a_s(m)} = Z_{\tilde{a}_s(m)}\} \). Then, define the matching event:
\begin{equation}
\mathcal{M}_s := \{\xi_s \leq M_s\}, \quad M_s := \left\lceil \frac{\log((s+1)^2/\delta)}{\log(1/(1-\gamma))} \right\rceil
\label{eq:matching_event}
\end{equation}
which ensures a successful coupling within \( M_s \) trials. Since \( \mathbb{P}(\tilde{a}_s(m) = N_s) = \gamma \), the number of trials \( M_s \) is selected to guarantee \( \mathbb{P}(\mathcal{M}_s) \geq 1 - \delta/(s+1)^2 \).

The hyperparameter \( \gamma \) controls the trade-off: as \( \gamma \) increases, the probability of coupling success increases (thus requiring fewer resampling trials), while the size of the regularization set \( h_t \) must increase. Table~\ref{tab:coupling} illustrates the data structures for successful and failed couplings during resampling.

\begin{table}[t]
    \centering
    \begin{tabular}{c|c|c||c|c|c}
    \toprule & \multicolumn{2}{c||}{Hypothetical Bandit Problem} & \multicolumn{3}{c}{Original Bandit Problem} \\
    \hline
     & Arm 1 & Arm 2  & Arm 1 & Arm 2 & Arm 3   \\
    \hline
    Failure &  \cellcolor{Gray}$(Z_{\tilde{a}_t,t}, W_{\tilde{a}_t,t})$  & $(Z_{N_t,t},?)$ & $(X_{1,t},?)$  & $(X_{2,t},?)$ & \cellcolor{Gray}$(X_{a_t,t},Y_{a_t,t})$ \\
    \hline
    Success &  $(Z_{1,t}, ?)$  & \cellcolor{Gray}$(Z_{N_t,t},W_{N_t,t})$ & $(X_{1,t},?)$ & \cellcolor{Gray}$(X_{a_t,t},Y_{a_t,t})$   & $(X_{3,t},?)$ \\
    \bottomrule
    \end{tabular}
    \caption{Illustration of coupling success and failure during resampling for \(N_t = 2\) and \(K = 3\). By construction, \(Z_{2,t} := X_{a_t,t}\) and \(W_{2,t} := Y_{a_t,t}\). Gray cells indicate the selected actions.}
    \label{tab:coupling}
\end{table}

\begin{algorithm}[t]
\caption{Candidate-Arm Sampler (\texttt{CAS}) for Round $t$}
\label{alg:cas}
\begin{algorithmic}[1]
\STATE \textbf{Input:} contexts $\{X_{k,t}\}_{k\in[K]}$, posterior mean $\widehat{\theta}_{t-1}$, exploration variance $v_{t-1}$, Gram matrix $V_{t-1}$, pseudo-index $N_t$, coupling parameter $\gamma$, confidence $\delta$.
\STATE \textbf{Set} $M_t$ as in~\eqref{eq:matching_event} \hfill \COMMENT{// maximum retries}
\STATE \textbf{Initialise} $m\gets1$
\REPEAT
  \STATE Sample $\tilde{\theta}^{(m)}_{k,t} \sim \mathcal{N}\!\bigl(\widehat{\theta}_{t-1},\,v_{t-1}^{2}V_{t-1}^{-1}\bigr)$ independently for all $k\in[K]$
  \STATE $a_t^{(m)} \gets \arg\max_{k\in[K]} X_{k,t}^{\top}\tilde{\theta}^{(m)}_{k,t}$
  \STATE Sample $\tilde{a}_t^{(m)}$ from the distribution in~\eqref{eq:pseudo_prob}
  \STATE $m \gets m+1$
\UNTIL{$\tilde{a}_t^{(m-1)} = N_t \;\;\mathbf{or}\;\; m> M_t$}
\STATE \textbf{Output:} $a_t^{\star} \gets a_t^{(m-1)}$, \quad $\tilde{a}_t^{\star}\gets\tilde{a}_t^{(m-1)}$
\end{algorithmic}
\end{algorithm}

The proposed resampling-coupling scheme is describe in Algorithm~\ref{alg:cas} as candidate arm sampler (\texttt{CAS}).
The resampling in \texttt{CAS} is distinctive from that in \citet{kim2021doubly} and \citet{xu2020upper}.
The resampling in \citet{kim2021doubly} resamples the action to find the arm whose selection probability is greater than a prespecified threshold value.
\citet{xu2020upper} resamples the previous counterfactual actions and contexts to impose randomization and generalization on the estimator.
In contrast, our resampling is to couple the hypothetical samples with the original samples and this coupling is the first method a novel innovative part in this work.

Upon obtaining the coupled contexts \(Z_{\tilde{a}_s(M_s),s} = X_{a_s(M_s),s}\), we construct the coupled pseudo-reward as:
\begin{equation}
W_{i,s}^{Co(\check{\theta}_t)} := \left(1 - \frac{\mathbb{I}(\tilde{a}_s(M_s) = i)}{\phi_{i,s}}\right) Z_{i,s}^\top \check{\theta}_t + \frac{\mathbb{I}(\tilde{a}_s(M_s) = i)}{\phi_{i,s}} W_{i,s},
\label{eq:CoY}
\end{equation}
which is computable for all \(i \in [N_s]\) because \(a_s(M_s)\) is selected and \(W_{\tilde{a}_s(M_s),s} = Y_{a_s(M_s),s}\) is observable.

Given an reference estimator \(\check{\theta}_t\) defined in~\eqref{eq:impute}, the proposed hypothetical coupled sample augmented (HCSA) estimator is defined as:
\begin{equation}
\widehat{\theta}_t := \left\{\sum_{s=1}^t \mathbb{I}(\mathcal{M}_s) \sum_{i=1}^{N_s} Z_{i,s} Z_{i,s}^\top \right\}^{-1} \left(\sum_{s=1}^t \mathbb{I}(\mathcal{M}_s) \sum_{i=1}^{N_s} W_{i,s}^{Co(\check{\theta}_t)} Z_{i,s}\right).
\label{eq:A_estimator}
\end{equation}
The indicator $\II(\Mcal_s)$ lets the estimator use the coupled pseudo-rewards~\eqref{eq:CoY} only when \(\mathcal{M}_s\) occurs; otherwise, it skips round \(s\) and relies on the previous estimator. Since \(\mathcal{M}_s\) occurs with high probability, we can couple the HCSA estimator with the hypothetical sample augmented estimator from~\eqref{eq:Hypo_DR}.

While the DR estimator in~\eqref{eq:DR_estimator} where the $K$ pseudo-rewards are augmented, the proposed (HCSA) estimator~\eqref{eq:Hypo_DR} adds $N_s \le d+1$ for each round $s\in[t]$.
This reduction in the number of augmented pseudo-reward samples paves a way to reduce the error and eliminate the IID and minimum eigenvalue assumption on contexts, by which \citet{kim2021doubly} used to obtain a regret bound that depends on the minimum eigenvalue of the context covariance.

Next, we provide a coupling inequality that relates the HCSA estimator to hypothetical sample augmented estimator $\tilde{\theta}_t^{H(\check{\theta}_t)}$.
\begin{lemma}[A coupling inequality]
\label{lem:coupling}  
For \( t \geq 1 \), let \( \Scal_{t} := \cap_{s=1}^{t} \Mcal_s \), where $\Mcal_s$ is the matching event defined in~\eqref{eq:matching_event}.
For the reference estimator \( \check{\theta}_t\) defined in~\eqref{eq:impute} and for \( x > 0 \),  
\[
\PP\left(\left\{ \norm{\widehat{\theta}^{A(\check{\theta}_t)}_{t}-\theta_{\star}}_{V_{t}} > x \right\}\right) \leq  
\PP\left(\left\{\norm{\tilde{\theta}_t^{H(\check{\theta}_t)}-\theta_{\star}}_{V_t} > x\right\} \cap \Scal_{t}\right) + \PP(\Scal_t^c),
\]  
and the failure probability satisfies \( \PP(\Scal_t^c) \leq \delta \).
\end{lemma}

\begin{proof}
Fix \( t \in [T] \) throughout the proof. 
For any \( \check{\theta} \in \RR^d \) and \( x > 0 \), decompose the probability as follows:  
\[
\PP\left(\norm{\widehat{\theta}_{t}-\theta_{\star}}_{V_t} > x \right) 
\leq \PP\left(\left\{\norm{\widehat{\theta}_{t}-\theta_{\star}}_{V_t} > x\right\} \cap \Scal_t \right) + \PP\left(\Scal_t^c\right).
\]
On the event \( \Scal_{t} := \cap_{s=1}^{t} \Mcal_s \), the HCSA estimator in~\eqref{eq:A_estimator} is simplified to
\[
\widehat{\theta}^{A(\check{\theta}_t)}_t = \left\{\sum_{s=1}^{t}\sum_{i=1}^{N_s} Z_{i,s}Z_{i,s}^\top \right\}^{-1}\left(\sum_{s=1}^{t}\sum_{i=1}^{N_s} W_{i,s}^{Co(\check{\theta}_t)}Z_{i,s}\right) = V_t^{-1}\left(\sum_{s=1}^{t}\sum_{i=1}^{N_s} W_{i,s}^{Co(\check{\theta}_t)}Z_{i,s}\right).
\]  
Define the function  
\[
F\Bigl(\tilde{a}_1(M_1),\ldots,\tilde{a}_t(M_t)\Bigr) := \norm{\widehat{\theta}_t-\theta_{\star}}_{V_t} 
= \norm{\sum_{s=1}^{t}\sum_{i=1}^{N_s} (W_{i,s}^{Co(\check{\theta}_t)}-Z_{i,s}^\top \theta_{\star} )Z_{i,s} }_{V_t^{-1}}.
\]
Using the definition of \( M_s \), where \( \tilde{a}_s(M_s) = N_s \), we have  
\begin{align*}
&\PP\left(\left\{F\Bigl(\tilde{a}_1(M_1),\ldots,\tilde{a}_t(M_t)\Bigr) > x\right\} \cap \Scal_t\right) \\
&= \PP\left(\left\{F\Bigl(\tilde{a}_1(M_1),\ldots,\tilde{a}_t(M_t)\Bigr) > x\right\} \cap \Scal_t \cap \bigcap_{s=1}^{t}\{\tilde{a}_s(M_s) = N_s\}\right) \\
&= \PP\left(\left\{F\Bigl(\tilde{a}_1(1),\ldots,\tilde{a}_t(1)\Bigr) > x\right\} \cap \Scal_t \cap \bigcap_{s=1}^{t}\{\tilde{a}_s(1) = N_s\}\right),
\end{align*}  
where the last equality holds because \( \{\tilde{a}_s(m) : m \in \NN\} \) are IID for each \( s \in [t] \).  
Then,  
\begin{align*}
&\PP\left(\left\{F\Bigl(\tilde{a}_1(1),\ldots,\tilde{a}_t(1)\Bigr) > x\right\} \cap \Scal_t \cap \bigcap_{s=1}^{t}\{\tilde{a}_s(1) = N_s\}\right) \\
&\leq \PP\left(\left\{F\Bigl(\tilde{a}_1(1),\ldots,\tilde{a}_t(1)\Bigr) > x\right\} \cap \Scal_t\right)
\end{align*} 
We observe that replacing $\{\tilde{a}_s(M_s):s\in[t]\}$ in coupled pseudo-rewards~\eqref{eq:CoY} with $\{\tilde{a}_s(1):s\in[t]\}$ gives the hypothetical pseudo-rewards in~\eqref{eq:HDRY}.
Thus, the distribution of the normalized error $\|\tilde{\theta}_t^{H(\check{\theta}_t)} - \theta_{\star}\|_{V_t}$ is equivalent to that of $F\bigl(\tilde{a}_1(1),\ldots,\tilde{a}_t(1)\bigr)$ and we obtain
\[
\PP\Big(\big\{F\bigl(\tilde{a}_1(1),\ldots,\tilde{a}_t(1)\bigr) > x\big\} \cap \Scal_t\Big) = \PP\Big(\big\{\|\tilde{\theta}_t^{H(\check{\theta}_t)} - \theta_{\star}\|_{V_t} > x\big\} \cap \Scal_t\Big),
\]
which proves the coupling inequality.
The bound for the failure probability $\PP(\Scal^c_t) \le \delta$ is proved by the fact that $\PP(\Mcal_s^c)\le \delta/(s+1)^2$ by construction of the maximum number of resampling trials.
\end{proof}
With the coupling inequality, we can leverage augmented samples from the hypothetical bandit problem to closely approximate the hypothetical sample augmented estimator with high probability. While the coupling technique can be applied to any choice of hypothetical problem—and the hypothetical contexts \( \{Z_{i,t}\} \) may be arbitrarily defined -- it is crucial that they are compatible with the original contextual bandit problem and the reference estimator \( \check{\theta}_t \).

The design of suitable hypothetical contexts is key, as it enables control of the maximum deviation in the original problem through the bound:
\[
\max_{k \in [K]} \left|X_{k,t}^\top \left(\tilde{\theta}_t^{H(\check{\theta}_t)} - \theta_\star\right)\right| 
\le 
\norm{\tilde{\theta}_t^{H(\check{\theta}_t)} - \theta_\star}_{\tilde{G}_t} 
\cdot 
\max_{k \in [K]} \norm{X_{k,t}}_{\tilde{G}_t^{-1}},
\]
where \( \tilde{G}_t := \sum_{s=1}^{t} \sum_{i=1}^{N_s} Z_{i,s} Z_{i,s}^\top \) denotes the Gram matrix constructed from the hypothetical contexts.

This upper bound consists of two components: (i) the self-normalized error of the hypothetical sample augmented estimator, and (ii) the maximum norm of the original contexts normalized by \( \tilde{G}_t \). Each component is sensitive to how the Gram matrix \( \tilde{G}_t \) is constructed.

If \( \tilde{G}_t \) is defined using only the played contexts -- i.e., \( \tilde{G}_t = \sum_{s=1}^{t} X_{a_s,s} X_{a_s,s}^\top + I_d \)—then the norm term \( \max_{k \in [K]} \|X_{k,t}\|_{\tilde{G}_t^{-1}} \) may become unbounded due to insufficient exploration. Conversely, if \( \tilde{G}_t \) is constructed from an overly large set of hypothetical contexts, the self-normalized error \( \| \tilde{\theta}_t^{H(\check{\theta}_t)} - \theta_\star \|_{\tilde{G}_t} \) may increase significantly, making it harder to guarantee tight estimation bounds.

Therefore, the design of the hypothetical contexts must carefully balance the number of augmenting samples to ensure both the estimator's accuracy and the tightness of high-probability bounds. Section~\ref{sec:self} presents a formal analysis of this trade-off, and shows that the proposed construction achieves this balance effectively, leading to a well-conditioned estimator with high-probability regret guarantees.

\subsection{Hypothetical Coupled Sample Augmented Thompson Sampling}
\label{sec:algorithm}

\begin{algorithm}[t]
\caption{Hypothetical Coupled Sample Augmented Thompson Sampling (\texttt{HCSA+TS})}
\label{alg:ATS} 
\begin{algorithmic}[1]
\STATE \textbf{Input:} confidence level $\delta\in(0,1)$, coupling parameter $\gamma \in (0,1)$, exploration parameter $v_{t} := \{2\log\frac{K(t+1)^2}{\delta}\}^{-1/2}$, orthogonal basis regularization parameter $h_t$ as in~\eqref{eq:h_t}.
\STATE Initialize the estimator $\widehat{\theta}_0=\mathbf{0}$, Gram matrix $V_{0}=O$, and a subset of rounds for the orthogonal basis regularization $\mA_{0}=\emptyset$. 
\FOR{$t=1$ \textbf{to} $T$} 
\STATE Observe contexts $\{X_{k,t}:k\in[K]\}$.
\STATE Update $\Acal_t$ as in \eqref{eq:A} and compute $N_t$ as in \eqref{eq:N}.
\STATE Set $m=1$ sample $\tilde{a}_t(m)$ from the multinomial distribution \eqref{eq:pseudo_prob}.
\STATE $(a_t(m),\tilde{a}_t(m)) \gets \texttt{CAS}\bigl(\{X_{k,t}\},\widehat{\theta}_{t-1},v_{t-1},V_{t-1},N_t,\gamma,\delta\bigr)$
\IF{$\tilde{a}_t(m) = N_t$ // Resampling succeeded} 
\STATE Pull arm $a_t(m)$ and observe $Y_{a_t(m),t}$.
\STATE Compute the reference estimator $\check{\theta}_t$ defined in~\eqref{eq:impute}.
\IF{$t \in \Acal_t$}
\STATE Compute hypothetical contexts $\{Z_{i,t}:i\in[N_t]\}$ as in \eqref{eq:new_contexts}
\ELSE
\STATE Compute hypothetical contexts $\{Z_{i,t}:i\in[N_t]\}$ with orthogonal basis as in \eqref{eq:new_contexts}.
\ENDIF
\STATE Update $V_{t}=V_{t-1}+\sum_{i=1}^{N_t}Z_{i,s}Z_{i,s}^{\top}$
\STATE Compute the estimator $\widehat{\theta}_t$ as in ~\eqref{eq:A_estimator}.
\ELSE
\STATE $\widehat{\theta}_{t} \leftarrow \widehat{\theta}_{t-1}$
\ENDIF
\ENDFOR 
\end{algorithmic} 
\end{algorithm}

The proposed algorithm, \textit{Hypothetical Coupled Sample Augmented Thompson Sampling} (\texttt{HCSA+TS}), is detailed in Algorithm \ref{alg:ATS}. This algorithm builds upon the structure of \texttt{LinTS} but introduces two key innovations: (i) resampling to couple the hypothetical bandit with the original bandit, and (ii) the HCSA estimator, which leverages a compressed orthogonal basis for improved efficiency.

For the resampling step (i), the algorithm repeatedly resamples $a_t$ from \texttt{LinTS} policy equipped with HCSA estimator and the novel Gram matrix until the condition \(\{Z_{\tilde{a}_t,t} = X_{a_t,t}\}\) is satisfied. 
This ensures that the hypothetical bandit aligns with the original bandit by augmenting the randomized contexts. The number of resampling attempts, set to \(\lceil\log\frac{(t+1)^2}{\delta} / \log\frac{1}{1-\gamma}\rceil\), guarantees the resampling process succeeds with probability at least \(1 - \delta / (t+1)^2\). 
In practice, this resampling typically succeeds after only a few iterations.

For the HCSA estimator (ii), although computing it might appear computationally demanding, efficient implementation strategies significantly reduce its complexity. 
Theoretically, the algorithm must compute hypothetical context vectors \(\{Z_{i,t} : i \in [N_t]\}\) for each resampling iteration. However, in practice, the algorithm first checks if \(\{\tilde{a}_t = N_t\}\) occurs and then computes the hypothetical contexts based on the resampled \(a_t\).

The worst-case computational complexity of the algorithm is \(O(d^2(K + d)T + T \log(T+1) / \log(\frac{1}{1-\gamma}))\). The primary computational bottleneck arises from calculating the Gram matrix \(V_t\) and performing eigenvalue decomposition to construct hypothetical contexts, repeated for the specified number of resampling attempts in each round \(t \in [T]\). 
In practice, the computational efficiency of the A estimator can be further enhanced by applying the Sherman-Morrison formula, enabling rank-1 updates and reducing memory usage.


\section{Regret Analysis}
\label{sec:regret_analysis}
The following theorem establishes a nearly minimax-optimal cumulative regret bound for the \texttt{HCSA+TS} algorithm.

\begin{theorem}[Regret Bound for \texttt{HCSA+TS}]
\label{thm:regret_bound}
In Algorithm~\ref{alg:ATS}, set the exploration parameter as \(v_{t} = \{2\log(K(t+1)^{2}/\delta)\}^{-1/2}\). Then, with probability at least \(1 - 3\delta\), the cumulative regret of the \texttt{HCSA+TS} algorithm by round \(T\) satisfies:
\begin{equation}
R(T) \leq \frac{4d}{(1/2 - e^{-1})(1 - \gamma)} \log\frac{d(T+1)^2}{\delta} 
+ \left\{5\sqrt{2} \theta_{\max} + \frac{6\sigma}{\gamma} \sqrt{2d \log\frac{2T}{\delta}} \right\} \sqrt{4dT \log \frac{T}{d}}.
\label{eq:regret_bound}
\end{equation}
\end{theorem}

The leading order term of the regret bound is \(O(d\sqrt{T} \log T)\), which matches the minimax lower bound \(\Omega(d\sqrt{T})\) established in \citet{lattimore2020bandit}, up to logarithmic factors. This result provides the first nearly minimax-optimal regret guarantee for linear contextual bandits under arbitrary context distributions.

For comparison, \cite{kim2021doubly} achieved $\tilde{O}(d\sqrt{T})$ regret bound under IID contexts with some special distributions of which the minimum eigenvalue of the average covariance matrix is $\Omega(1/d)$.
\citet{kim2023squeeze} achieved a regret bound of \(O(\sqrt{dT \log T})\) under the assumption of IID contexts with a strictly positive minimum eigenvalue for their average covariance matrix. Similarly, \citet{huix2023tight} obtained a regret bound of \(\tilde{O}(d\sqrt{T})\), but their analysis assumes a Gaussian prior on the unknown parameter \(\theta_{\star}\).

Earlier works such as \citet{kim2021doubly} and \citet{agrawal2013thompson} impose unit-norm assumptions on the contexts.
In contrast, the regret bound~\eqref{eq:regret_bound} does not require normalization of the context vectors. The only assumption made is that the absolute inner product between any context and the true parameter is bounded, i.e., \( |X_{k,t}^\top \theta_{\star}| \leq 1 \). Notably, increasing the context norm bound \(x_{\max}\) does not cause the regret to grow linearly, avoiding the scaling issues encountered in prior analyses.

The key technical contributions enabling this result include: (i) the development of a self-normalized bound for the HCSA estimator using a carefully constructed Gram matrix (Section~\ref{sec:self}), (ii) the identification of a set of low-regret arms selected with high probability (Section~\ref{sec:low_regret_arms}), and (iii) a novel maximal elliptical potential bound based on the augmented Gram matrix \(V_t\) (Section~\ref{sec:max_elliptical}).

\subsection{A Self-Normalized Bound for the Proposed Estimator}
\label{sec:self}
With the coupling inequality (Lemma~\ref{lem:coupling}), we can bound the error of the proposed estimator by obtaining an error bound for the hypothetical estimator, which is proven in the following lemma.

\begin{lemma}[A self-normalized bound of the HSA estimator]
\label{lem:error_decomposition}
For each $t\ge1$, define the matrix $A_t:=\sum_{s=1}^{t}\phi_{\tilde{a}_s,s}^{-1}Z_{\tilde{a}_s,s}Z_{\tilde{a}_s,s}^\top $.
and.
Then the self-normalized bound of the hypothetical sample augmented estimator is decomposed as:
\[
\norm{\tilde{\theta}_t^{H(\check{\theta}_t)}- \theta_{\star}}_{V_t} \le \norm{V_t^{-1/2}(V_t-A_t)(\check{\theta}_t-\theta_{\star})}_2 + \norm{\sum_{s=1}^{t} \phi_{\tilde{a}_s,s}^{-1}(W_{\tilde{a}_s,s}-Z_{\tilde{a}_s,s}^\top \theta_{\star})Z_{\tilde{a}_s,s}}_{V_t^{-1}}
\]
\end{lemma}

The proof is in Appendix~\ref{sec:error_decomposition_proof}.
The decomposition shows the two sources of error for the HDR estimator for rewards of all arms: (i) from the reference estimator \(\check{\theta}\) used as an reference estimator in HDR pseudo-rewards~\eqref{eq:HDRY}, and (ii) the noise error of the rewards. 
In error term (i), \(V_t^{-1/2}(V_t - A_t)\) is the matrix martingale with bounded eigenvalues, which is bounded by the newly developed matrix concentration inequality (Lemma \ref{lem:matrix_neg}).
The error term (ii) is bounded by a modified martingale inequality developed by \citet{abbasi2011improved}.

With the suitable choice of \(\check{\theta}_t\), we obtain an \(O(\sqrt{d \log t})\) error bound for $\tilde{\theta}_{t}^{H(\check{\theta}_t)}$ estimator, which is normalized by the novel augmented Gram matrix \(V_t\) instead of the conventional Gram matrix that includes only selected contexts.

\begin{theorem}[Self-Normalized Bound for the HCSA Estimator]
\label{thm:self}
With probability at least \(1 - 3\delta\), the estimator defined in~\eqref{eq:A_estimator} satisfies
\[
\norm{\widehat{\theta}_t - \theta_{\star}}_{V_t} \leq 5\theta_{\max} + \frac{6\sigma}{\gamma} \sqrt{d \log\frac{1 + t}{\delta}}
\]
for all \( t \geq T_1 \).
\end{theorem}
The proof is provided in Appendix~\ref{sec:tail_proof}. 
Unlike the classical self-normalized bound of \citet{abbasi2011improved}, which is normalized by the Gram matrix \( \sum_{s=1}^{t} X_{a_s,s} X_{a_s,s}^\top + I_d \) built solely from selected contexts, Theorem~\ref{thm:self} establishes a bound normalized by the full Gram matrix \( V_t \), which includes contexts from all \( K \) arms. 
While \citet{kim2023squeeze} also considered self-normalization with respect to a full Gram matrix, their estimator incorporates contexts from all arms only in a fraction of the rounds, and their analysis is restricted to IID contexts with a strictly positive-definite covariance matrix.
In contrast, Theorem~\ref{thm:self} applies to arbitrary (including non-IID, non-stationary) context sequences, establishing a uniform self-normalized bound with a full Gram matrix. This result enables a novel regret analysis of Thompson Sampling that yields a \(\tilde{O}(d\sqrt{T})\) bound for \texttt{HCSA+TS}, applicable under arbitral context distributions.

\subsection{Low-Regret Arms with a High-Probability Guarantee}
\label{sec:low_regret_arms}
For each \( k \in [K] \) and \( t \in [T] \), define the instantaneous regret gap between the optimal arm and arm \( k \) as
\[
\Delta_{k,t} := X_{a_t^{\star},t}^\top \theta_{\star} - X_{k,t}^\top \theta_{\star}.
\]
Using this, we define the set of low-regret arms at round \( t \) as
\begin{equation}
\mathcal{P}_{t} := \left\{ k \in [K] : \Delta_{k,t} \le 2x_t \norm{\widehat{\theta}_{t-1} - \theta_{\star}}_{V_{t-1}} + \sqrt{\norm{X_{a_t^{\star},t}}_{V_{t-1}^{-1}}^{2} + \norm{X_{k,t}}_{V_{t-1}^{-1}}^{2}} \right\},
\label{eq:low_regret_set}
\end{equation}
where \( x_t := \max_{k \in [K]} \|X_{k,t}\|_{V_{t-1}^{-1}} \).
The self-normalized confidence bound with respect to \( V_t \), which includes contexts beyond the selected arms, allows the construction of an effective set \( \mathcal{P}_t \) that has lower regret than arms in sets built from Gram matrices based solely on selected contexts such as \( \sum_{s=1}^{t} X_{a_s,s} X_{a_s,s}^\top + I_d \).

The following lemma provides a high-probability guarantee that the arm selected by Algorithm~\ref{alg:ATS} belongs to the low-regret set.
\begin{lemma}[High-Probability Selection of Low-Regret Arms]
\label{lem:super_unsaturated_arms}
Let \( a_t \) be the arm selected by Algorithm~\ref{alg:ATS}, and let \( \mathcal{P}_t \) be the set defined in~\eqref{eq:low_regret_set}. If the exploration parameter is set as \( v_t = \{2\log(K(t+1)^2 / \delta)\}^{-1/2} \), then
\[
\mathbb{P}\left(a_t \in \mathcal{P}_t \mid \mathcal{H}_t\right) \ge 1 - \frac{\delta}{(t+1)^2}.
\]
\end{lemma}
The proof is deferred to Appendix~\ref{sec:low_regret_proof}.
In contrast to \citet{agrawal2013thompson}, where bounding the probability of selecting a saturated arm required setting \( v = \sqrt{9d \log(t/\delta)} \), thereby introducing a \(\sqrt{d}\)-scaling, Lemma~\ref{lem:super_unsaturated_arms} establishes that such dimensional dependence is unnecessary. 
By leveraging the structure of the HCSA estimator and its associated Gram matrix, the proposed approach guarantees high-probability selection of low-regret arms without incurring additional dependence of the variance parameter $v$ on the dimension \( d \)

\subsection{Maximal Elliptical Potential Bound}
\label{sec:max_elliptical}
While continuing the proof of the regret bound, we face a novel terms that needs to be analyzed.
By Lemma~\ref{lem:super_unsaturated_arms}, the algorithm selects arms from the low-regret set \(\Pcal_t\), i.e., \(a_t \in \Pcal_t\), with high probability. This leads to the following novel regret decomposition:

\[
\Regret{t} = \Diff{a_t}{t} \leq 2x_t \norm{\Estimator{t-1} - \theta_{\star}}_{V_{t-1}} + \sqrt{\norm{X_{a_t^{\star},t}}_{V_{t-1}^{-1}}^2 + \norm{X_{a_t,t}}_{V_{t-1}^{-1}}^2},
\]
for \(t \geq T_1\). Thus, the cumulative regret is bounded as follows:
\[
R(T) \leq 2h_T + \sum_{t\in[T]\setminus\Acal_T} \left\{ 2x_t \norm{\Estimator{t-1} - \theta_{\star}}_{V_{t-1}} + \sqrt{\norm{X_{a_t^{\star},t}}_{V_{t-1}^{-1}}^2 + \norm{X_{a_t,t}}_{V_{t-1}^{-1}}^2} \right\}.
\]
Because \(x_t := \max_{k\in[K]}\|X_{k,t}\|_{V^{-1}_{t-1}}\), we get:
\[
R(T) \leq 2h_T + \sum_{t\in[T]\setminus\Acal_T} \big( 2x_t \|\Estimator{t-1} - \theta_{\star}\|_{V_{t-1}} + \sqrt{2}x_t \big).
\]
Factoring out \(x_t\):
\[
R(T) \leq 2T_1 + \sum_{t=\in[T]\setminus\Acal_T} \big( 2 \|\Estimator{t-1} - \theta_{\star}\|_{V_{t-1}} + \sqrt{2} \big) x_t.
\]
Since we obtain $\|\Estimator{t-1} - \theta_{\star}\|_{V_{t-1}} = O(\sqrt{d \log t})$ from Theorem~\ref{thm:self}, we need a bound for $\sum_{t\in[T]\setminus\Acal_T} x_t$, which is in the following lemma.

\begin{lemma}[Maximal elliptical potential lemma]
\label{lem:elliptical}  
For \(x_t := \max_{k \in [K]} \|X_{k,t}\|_{V_{t-1}^{-1}}\), we have:  
\[
\sum_{t=T_1}^{T} x_t^2 \leq 2d \log\frac{T}{d}.
\]
\end{lemma}
The proof is in Appendix~\ref{sec:proof_elliptical}.
Previous elliptical lemmas only obtain the bounds the normalized norm of the selected contexts $X_{a_t,t}$, while we need a bound for $\max_{k\in{K]}}\|X_{k,t}\|_{V_t}$.
The maximal elliptical potential lemma bounds the normalized norm of the contexts for all $K$ arms, which previous analyses could not bound effectively. 
This is possible because we augmented suitable sample to obtain the Gram matrix \(V_t\) consisting of contexts from all $K$ arms.

\section{Experimental Results}
\label{sec:experiment}
This section is for evaluating the empirical performance of the proposed algorithm, \texttt{HCSA+TS}, against several benchmark algorithms for LinCB with simulated data. 
The benchmarks include \texttt{LinTS}, \texttt{LinUCB}, \texttt{DRTS} \citep{kim2021doubly}, \texttt{HyRan} \citep{kim2023squeeze}, and \texttt{Sup\texttt{LinUCB}} \citep{chu2011contextual}.

\begin{figure}[t]

\centering
\subfigure[Regret comparison ($d=10$, $K=20$)]{{\includegraphics[width=0.48\textwidth]{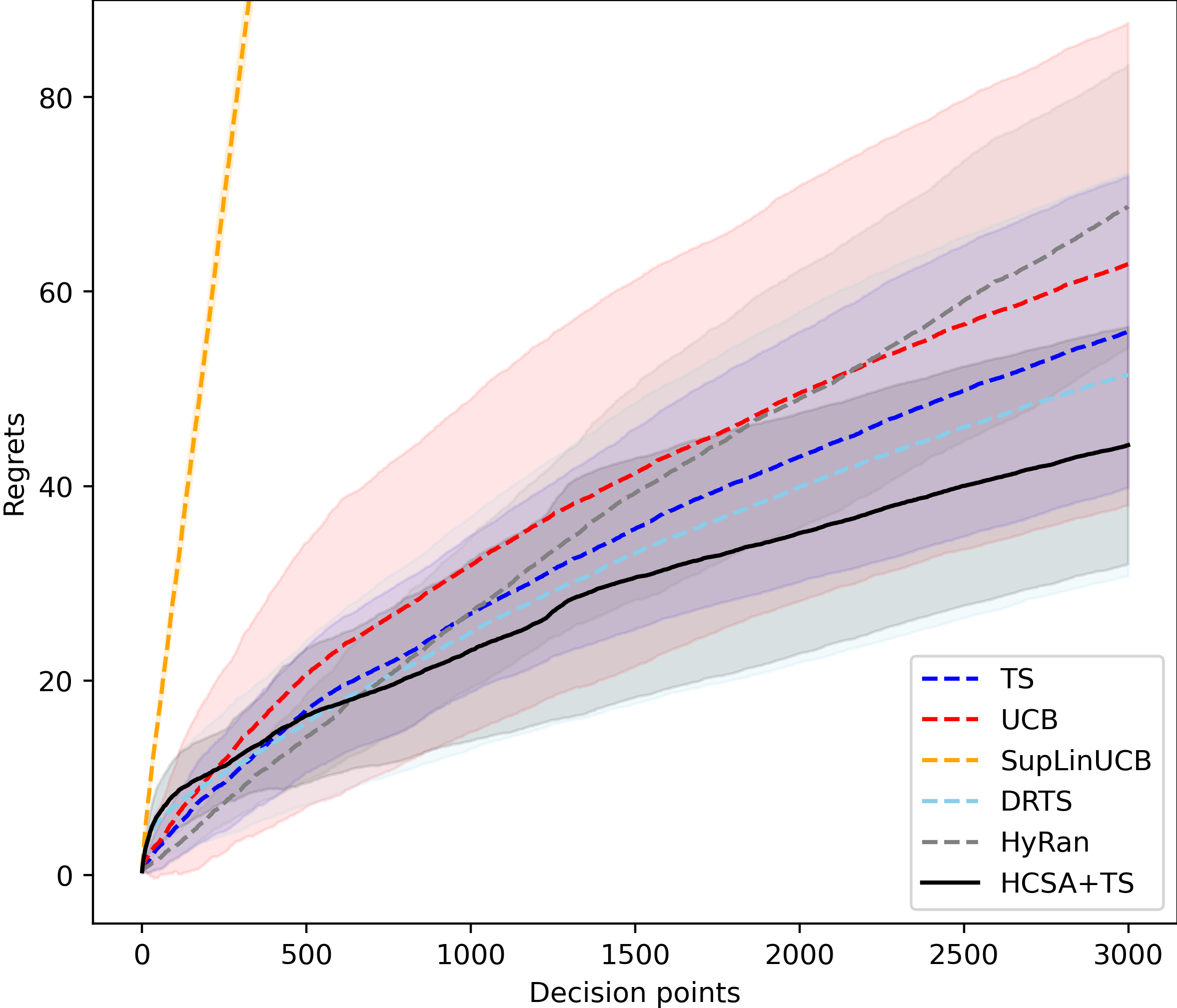}
}}
\subfigure[Regret comparison ($d=30$, $K=20$)]{{\includegraphics[width=0.48\textwidth]{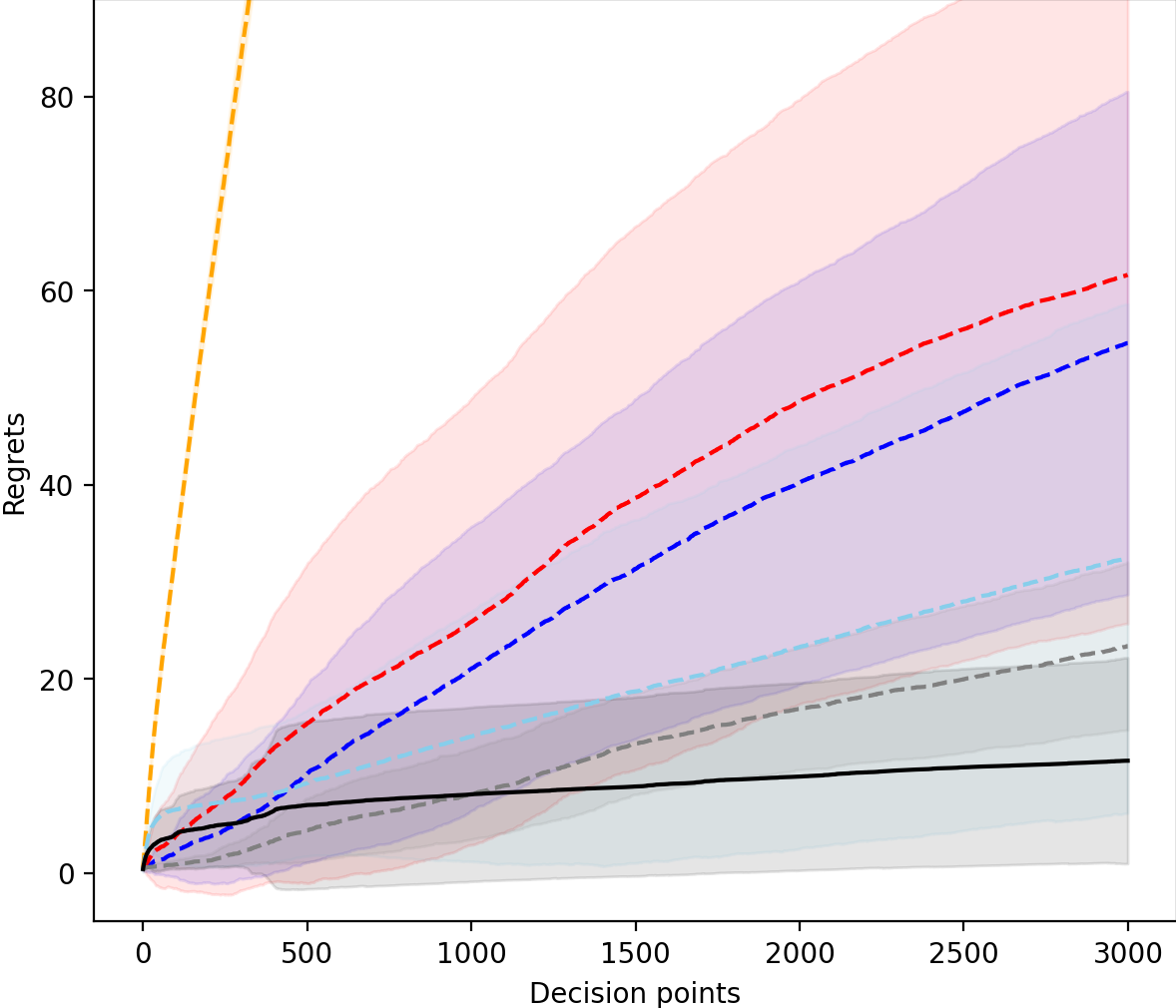}
}}
\\
\subfigure[Regret comparison ($d=10$, $K=30$)]{{\includegraphics[width=0.48\textwidth]{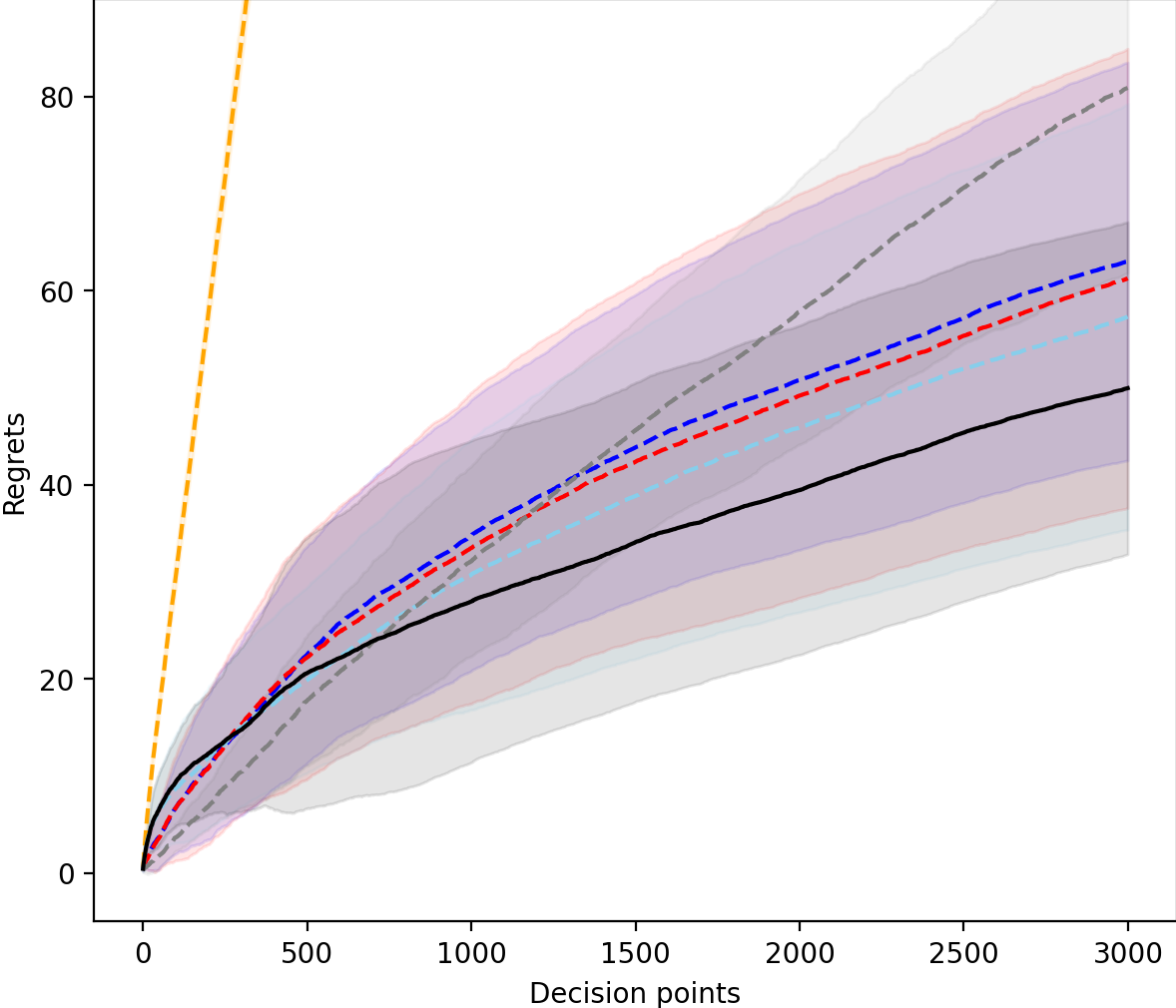}
}}
\subfigure[Regret comparison ($d=30$, $K=30$)]{{\includegraphics[width=0.48\textwidth]{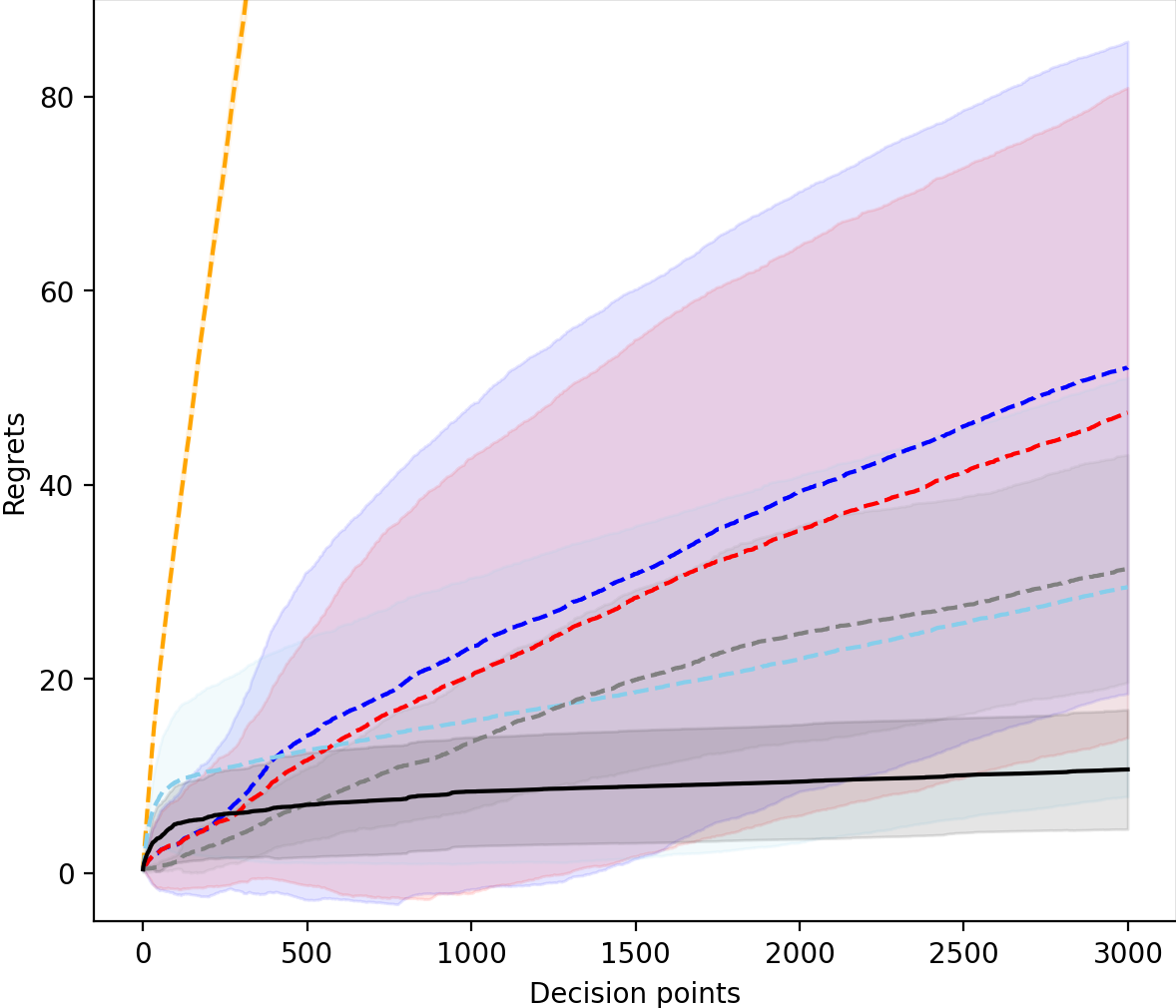}
}}
\caption{\label{fig:comparison} Comparison of the regrets of the proposed \texttt{HCSA+TS} algorithm with other benchmark methods.  
The lines represent the average, and the shaded areas indicate the standard deviation based on twenty experiments.  
The results demonstrate that the proposed \texttt{HCSA+TS} effectively identifies the optimal arm using orthogonal regularization.}
\end{figure}

\begin{figure}[t]
\centering
\subfigure[Prediction error comparison ($d=10$, $K=20$)]{{\includegraphics[width=0.48\textwidth]{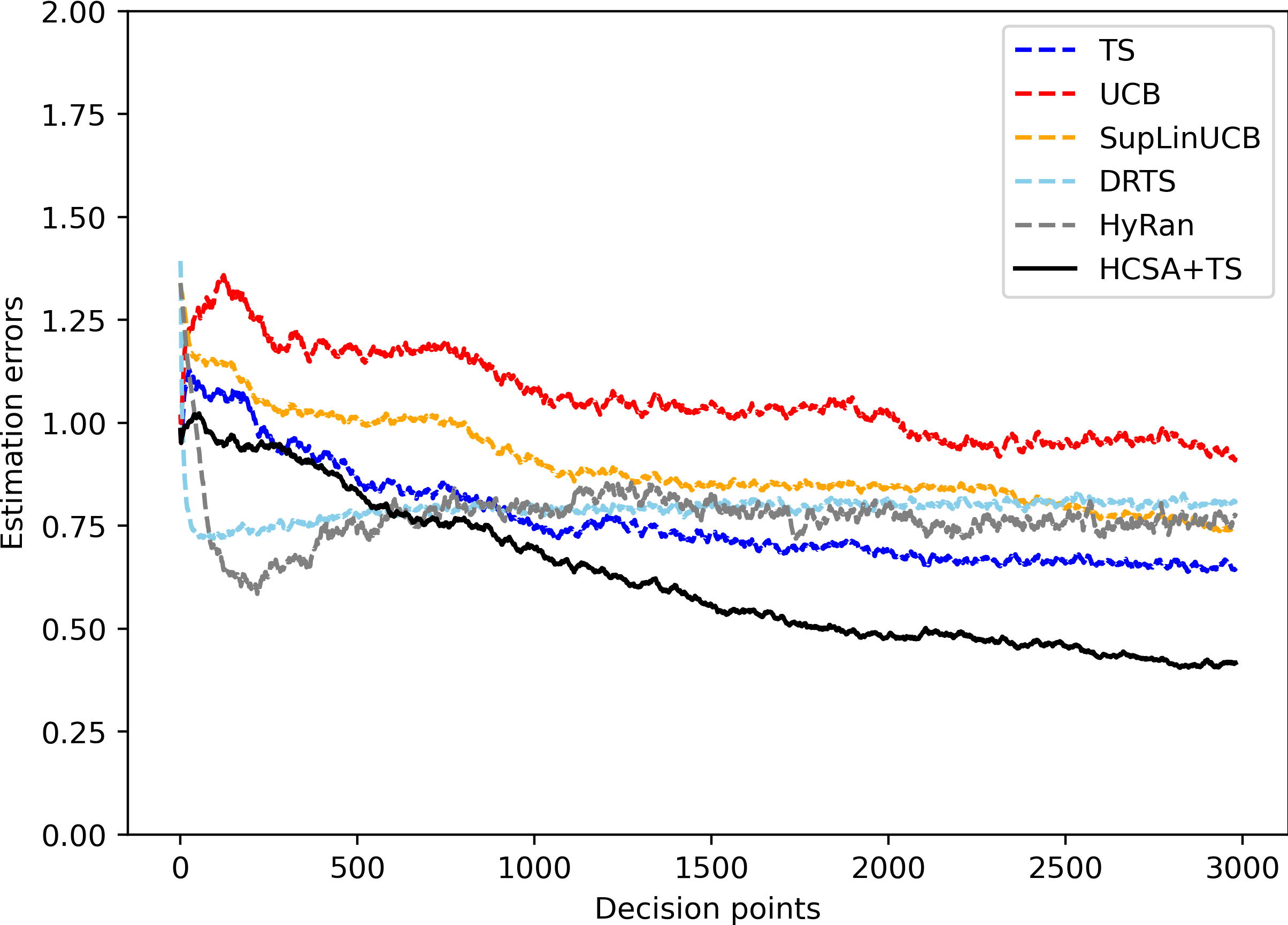}
}}
\hfill
\subfigure[Prediction error comparison ($d=30$, $K=20$)]{{\includegraphics[width=0.48\textwidth]{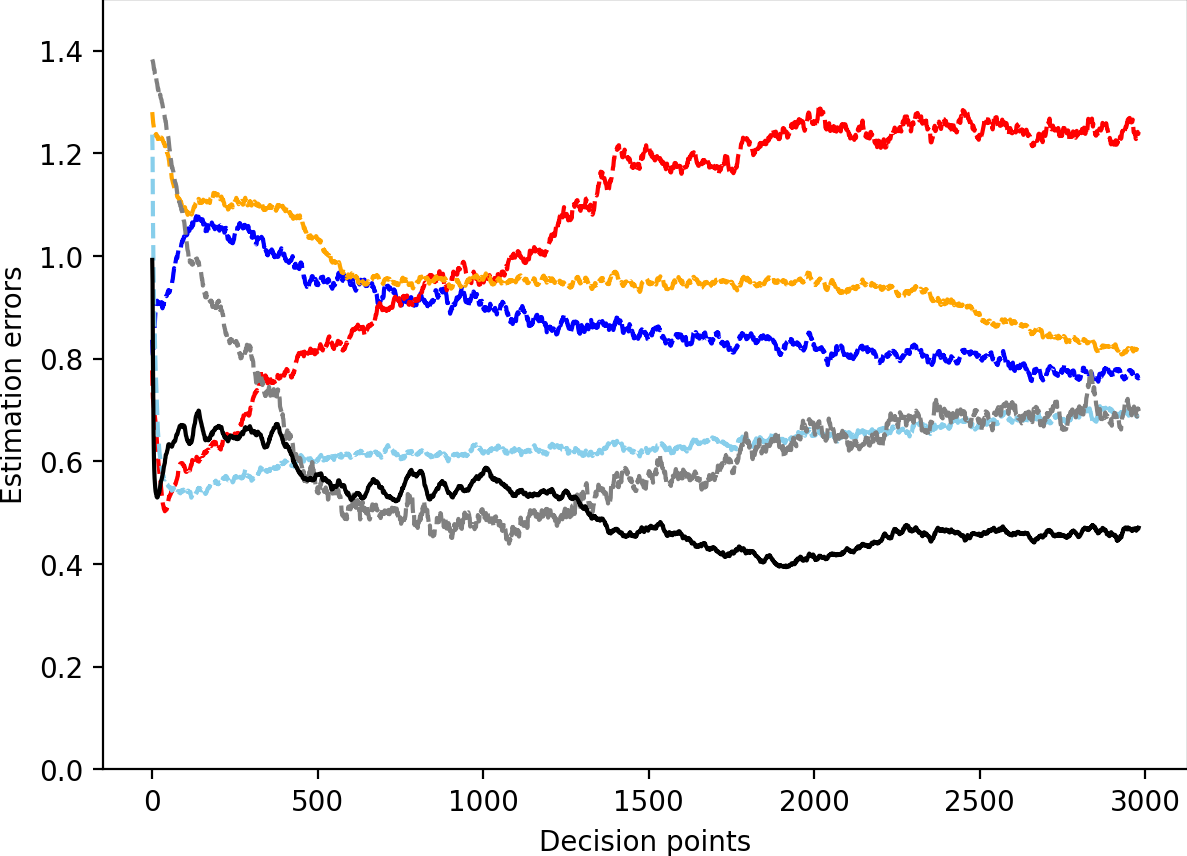}
}}
\\
\subfigure[Prediction error comparison ($d=10$, $K=30$)]{{\includegraphics[width=0.48\textwidth]{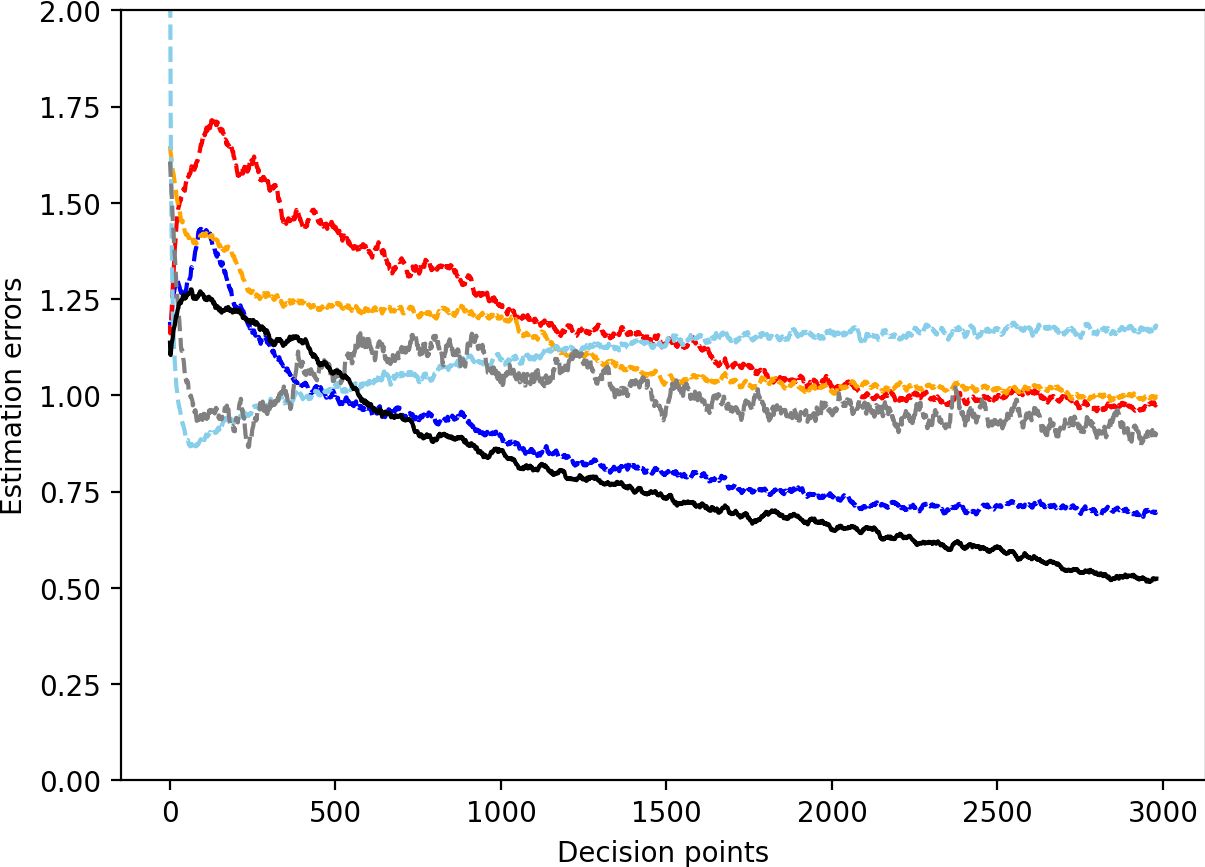}
}}
\hfill
\subfigure[Prediction error comparison ($d=30$, $K=30$)]{{\includegraphics[width=0.48\textwidth]{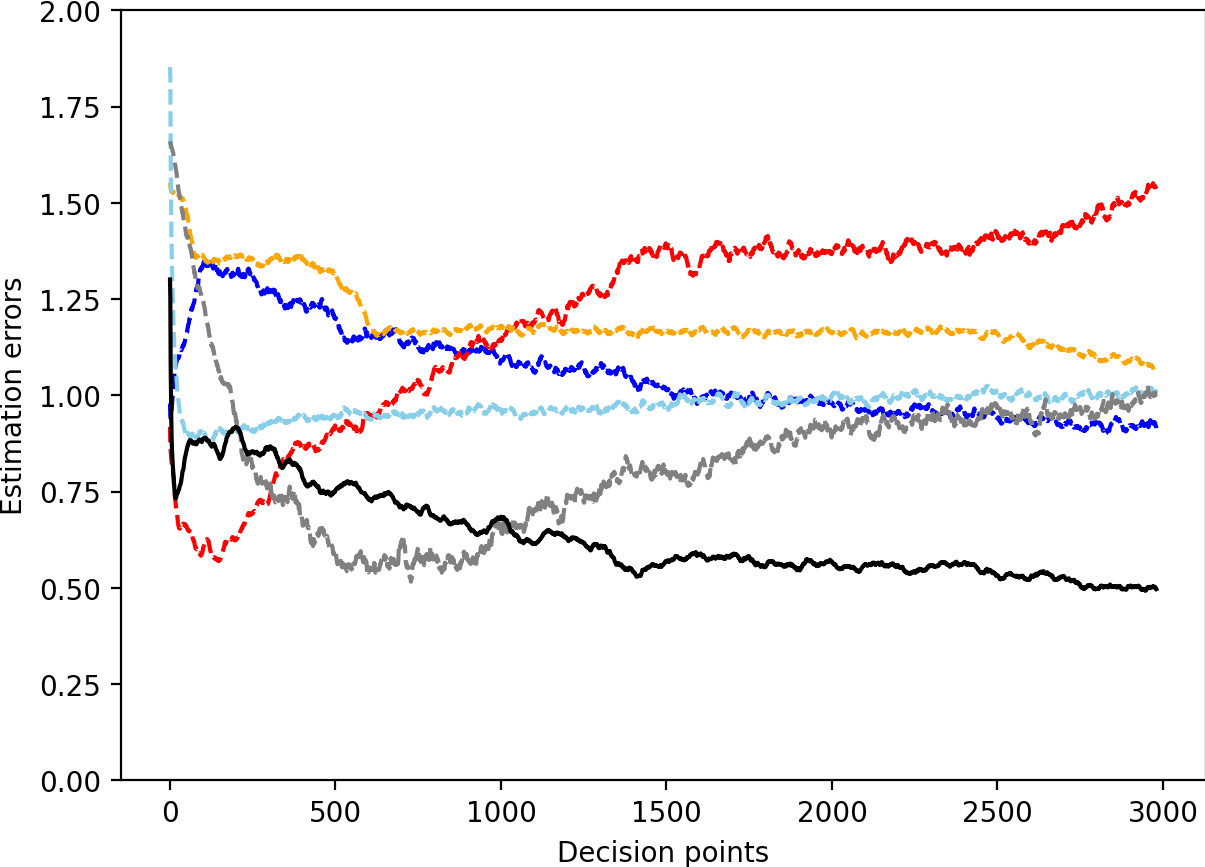}
}}
\caption{\label{fig:prediction} Comparison of the prediction error across all arms, calculated as \(\sum_{i=1}^{K}\{X_{i,t}^\top(\widehat{\theta}_t-\theta_{\star})\}^2\), for the proposed \texttt{HCSA+TS} and other benchmark methods.  
The lines represent the averages, and the shaded areas indicate the standard deviations based on twenty experiments.  
The results demonstrate that the proposed estimator, enhanced with orthogonal augmentation, learns the reward more accurately than other estimators.}
\end{figure}

For the experiment setting, the parameter \(\theta_{\star}\) is defined as:
\[
\theta_{\star} := \frac{1}{\sqrt{d}} \biggl( \underset{\lceil d/2 \rceil}{\underbrace{1, \cdots, 1}}, \underset{d - \lceil d/2 \rceil}{\underbrace{-1, \cdots, -1}} \biggr)^{\top},
\]
where \(d \in \{10, 30\}\) is the dimension of the parameter. 
The $i$-th entry of the context vectors for $K\in\{20,30\}$ arms are independently sampled from a Gaussian distribution with a mean of \(-1 + \frac{3(i-1)}{d-1}\) and variance of 1 for each \(i \in [d]\). 
These vectors are normalized and then scaled by a scalar drawn uniformly from \([0,1]\). 
To simulate missing context information, with probability \(1/2\), the last \(d - \lceil d/2 \rceil\) entries of the context vectors are set to zero at each round. This setting reflects practical scenarios where certain context features may be unavailable with some probability, making it challenging to estimate the corresponding entries in \(\theta_{\star}\).

The hyperparameter optimization was conducted as follows: For \texttt{LinTS}, the variance parameter was selected from \(\{0.01, 0.1, 1\}\). 
For \texttt{LinUCB} and \texttt{Sup\texttt{LinUCB}}, the confidence bound inflation parameter was chosen from \(\{0.01, 0.1, 1\}\). 
For \texttt{HyRan} and \texttt{HCSA+TS}, the regularization parameters \(p\) and \(\gamma\) were tuned from \(\{0.1, 0.5, 0.9\}\). 
The hyperparameters for \texttt{DRTS} were fixed as specified in \citet{kim2021doubly}. 
Values outside the specified ranges showed negligible differences in performance, suggesting robustness to hyperparameter selection for all methods.

Figure~\ref{fig:comparison} compares the cumulative regret of \texttt{HCSA+TS} with other benchmark algorithms across various configurations of \(d\) and the number of arms \(K\). Each line represents the average cumulative regret, and the shaded regions indicate the standard deviation across 20 independent trials.

The results show that \texttt{HCSA+TS} achieves the lowest cumulative regret in all tested settings. Compared to \texttt{LinTS}, \texttt{LinUCB}, and \texttt{Sup\texttt{LinUCB}}, which do not leverage information from all arms, \texttt{HCSA+TS} demonstrates robustness to missing context data. When compared to \texttt{DRTS} and \texttt{HyRan}, which uses the original context vectors, \texttt{HCSA+TS} consistently identifies low-regret arms more effectively, even under significant masking of context features. 

Initially, due to the orthogonal basis regularization, \texttt{HCSA+TS} incurs higher regret during the exploration phase, particularly when the effective rank of the context matrix is low. However, it rapidly adapts and identifies the optimal arm, ultimately outperforming the other algorithms, which continue to suffer regret due to their inability to handle missing context data effectively.

Figure~\ref{fig:prediction} illustrates the prediction error across all arms, measured as \(\sum_{i=1}^{K} \{X_{i,t}^\top (\widehat{\theta}_t - \theta_{\star})\}^2\). 
Similar to the regret results, the averages and standard deviations are computed over 20 trials.
The initial convergence of the estimators in \texttt{DRTS} and \texttt{HyRan} is faster due to their reliance on imputed contexts. However, their prediction errors increase over time because the imputed contexts, often containing many zero entries, provide incomplete information and hinder accurate estimation. In contrast, \texttt{HCSA+TS} demonstrates steady and consistent convergence throughout the experimental horizon. Its orthogonal augmentation strategy allows it to extract useful information even when parts of the context vectors are masked, leading to superior prediction accuracy compared to both traditional ridge-based estimators (\texttt{LinTS}, \texttt{LinUCB}, \texttt{Sup\texttt{LinUCB}}) and other augmented methods (\texttt{DRTS}, \texttt{HyRan}).

In summary, the experiments validate that \texttt{HCSA+TS} achieves significant improvements in both cumulative regret and prediction accuracy over existing benchmarks. Its ability to handle missing context information effectively while leveraging orthogonal regularization makes it particularly well-suited for practical scenarios with incomplete data.

\section{Conclusion}
\label{sec:conclusion}
This work introduces novel resampling and coupling techniques for integrating adaptive data augmentation into linear contextual bandits (LinCB), achieving a nearly minimax-optimal regret bound without imposing structural assumptions on the context vectors. By leveraging these techniques, the proposed approach enables the algorithm to effectively utilize contextual information from all arms as if full reward feedback were available in each round, while minimizing the number of required augmentations.
Beyond improving regret performance, this methodology marks a conceptual advancement in reward estimation. Through the use of hypothetical contexts and coupling-based resampling, we show that accurate estimation for rewards of all arms is feasible even under arbitrary context distributions, thereby broadening the applicability of LinCB methods in non-IID and adversarial settings.

The proposed framework offers a foundation for extending adaptive augmentation techniques to more general bandit models and reinforcement learning environments. In particular, these ideas have the potential to reduce variance in reward estimates for optimal arms or general policies, thus enabling more efficient exploration and decision-making in complex and high-dimensional settings. Future work may explore such extensions, including applications to nonlinear bandits, structured decision-making problems, and policy optimization in reinforcement learning, further amplifying the impact of the techniques developed in this study.

\acks{This work was supported by the Institute of Information \& Communications Technology Planning \& Evaluation (IITP) grant funded by the Korea government (MSIT) [RS-2021-II211341, Artificial Intelligence Graduate School Program (Chung-Ang University) and the Chung-Ang University Research Grants in 2025.
Wonyoung Kim also appreciates the proofreading and advising provided by Myunghee Cho Paik.
}

\appendix
\section{Missing Proofs}
\subsection{Proof of Lemma~\ref{lem:error_decomposition}}
\label{sec:error_decomposition_proof}
\begin{proof}
Recall that
\begin{align*}
    V_t &:= \sum_{s=1}^t \sum_{i=1}^{N_s}Z_{i,s}Z_{i,s}^\top, \\
    A_t &:= \sum_{s=1}^{t} \frac{1}{\phi_{\tilde{a}_s,s}} Z_{\tilde{a}_s,s}Z_{\tilde{a}_s,s}^\top = \sum_{s=1}^{t} \sum_{i=1}^{N_s}\frac{\II(\tilde{a}_s=i)}{\phi_{i,s}} Z_{i,s}Z_{i,s}^\top.
\end{align*}
By definition of the estimator, 
\[
\begin{split} 
& \norm{\tilde{\theta}^{H(\check{\theta}_t)}_{t}-\theta_{\star}}_{V_{t}}\\
& = \norm{\sum_{s=1}^{t} \sum_{i=1}^{r_{s}+1}\left(\tilde{W}_{i,s}^{H(\check{\theta}_{t})}-Z_{i,s}^{\top}\theta_{\star}\right)Z_{i,s}}_{V_{t}^{-1}}\\
& = \norm{\sum_{s=1}^{t} \sum_{i=1}^{r_{s}+1} \left\{ 1-\frac{\II(\tilde{a}_{s}=i)}{\phi_{i,s}} \right\} Z_{i,s}Z_{i,s}^\top (\check{\theta}_t-\theta_{\star}) + \frac{\II(\tilde{a}_{s}=i)}{\phi_{i,s}} \left(W_{i,s}-Z_{i,s}^{\top}\theta_{\star}\right)Z_{i,s}}_{V_{t}^{-1}}\\
& = \norm{(V_{t}-A_{t})(\check{\theta}_t-\theta_{\star}) + \sum_{s=1}^{t} \sum_{i=1}^{r_{s}+1} \frac{\II(\tilde{a}_{s}=i)}{\phi_{i,s}} \left(W_{i,s}-Z_{i,s}^{\top}\theta_{\star}\right)Z_{i,s}}_{V_{t}^{-1}},
\end{split}
\]
where the second equality holds by the definition~\eqref{eq:HDRY}.
Recall that 
\[
\]
Then, by the triangle inequality, we have
\[
\norm{\tilde{\theta}^{H(\check{\theta}_t)}_{t}-\theta_{\star}}_{V_{t}} 
\le \norm{(V_{t}-A_{t})(\check{\theta}_t-\theta_{\star})}_{V_{t}^{-1}} + \norm{\sum_{s=1}^{t} \sum_{i=1}^{r_{s}+1} \frac{\II(\tilde{a}_{s}=i)}{\phi_{i,s}} \left(W_{i,s}-Z_{i,s}^{\top}\theta_{\star}\right)Z_{i,s}.}_{V_{t}^{-1}}\!\!\!.
\]
Because
\[
\norm{(V_{t}-A_{t})(\check{\theta}_t-\theta_{\star})}_{V_{t}^{-1}} 
 = \norm{V_t^{-1/2}(V_{t}-A_{t})(\check{\theta}_t-\theta_{\star})}_{2},
\]
which completes the proof.
\end{proof}

\subsection{Proof of Theorem \ref{thm:self}}
\label{sec:tail_proof}

\begin{proof}
By the definition of the proposed estimator and the coupling inequality (Lemma~\ref{lem:coupling}), for any \( x > 0 \),
\[
\PP\left(\norm{\widehat{\theta}_t - \theta_{\star}}_{V_t} > x\right)
\le \PP\left(\left\{\norm{\widehat{\theta}^{HDR(\check{\theta}_t)}_t - \theta_{\star}}_{V_t} > x\right\} \cap \Scal_t\right) + \delta.
\]
From the error decomposition in Lemma~\ref{lem:error_decomposition},
for \( x > 0 \),
\begin{align*}
&\PP\left(\left\{\norm{\widehat{\theta}^{HDR(\check{\theta}_t)}_t - \theta_{\star}}_{V_t} > x\right\} \cap \Scal_t\right) \\
&\le \PP\left(\left\{\norm{V_t^{-1/2}(A_t - V_t)(\check{\theta}_t - \theta_{\star})}_{2} + \norm{S_t}_{V_t} > x\right\} \cap \Scal_t\right),
\end{align*}
where
\[
S_{t} := \sum_{s=1}^{t} \sum_{i=1}^{r_{s}+1} \frac{\II(\tilde{a}_{s}=i)}{\phi_{i,s}} \left(W_{i,s} - Z_{i,s}^{\top} \theta_{\star}\right)Z_{i,s}.
\]
By the definition of \( \check{\theta}_t \),
\[
\check{\theta}_t - \theta_{\star} := \left(\sum_{s=1}^{t} X_{a_s, s} X_{a_s, s} + \gamma  I_d\right)^{-1} \left\{ \sum_{s=1}^{t} (Y_{a_s, s} - X_{a_s, s}^{\top} \theta_{\star}) X_{a_s, s} - \gamma \theta_{\star} \right\},
\]
and similarly,
\[
\check{\theta}_t - \theta_{\star} = \left(A_t +  I_d\right)^{-1} \left(S_t -  \theta_{\star}\right).
\]
Substituting into the error bound, we have
\[
\begin{split}
\norm{\widehat{\theta}_t - \theta_{\star}}_{V_t} 
& \le \norm{(V_t - A_t) (A_t +  I_d)^{-1} (S_t -  \theta_{\star})}_{V_t^{-1}} + \norm{S_t}_{V_t^{-1}} \\
& = \norm{V_t^{-1/2}(V_t - A_t)(A_t +  I_d)^{-1} V_t^{1/2} (S_t -  \theta_{\star})}_{2} + \norm{S_t}_{V_t^{-1}},
\end{split}
\]
where the last equality holds by the definition of \( \|\cdot\|_{V_t^{-1}} \).
By the definition of the spectral norm \( \|\cdot\|_2 \),
\begin{align*}
&\norm{V_t^{-1/2}(V_t - A_t)(A_t +  I_d)^{-1} (S_t -  \theta_{\star})}_{2}\\
&\le \bigg\Vert\underset{P_t}{\underbrace{V_t^{-1/2}(V_t - A_t)(A_t +  I_d)^{-1} V_t^{1/2}}}\bigg\Vert_2 \norm{S_t -  \theta_{\star}}_{V_t^{-1}}.
\end{align*}
Thus, the bound becomes
\[
\norm{\widehat{\theta}_t - \theta_{\star}}_{V_t}  \le \big(\norm{P_t}_2+1\big)\norm{S_t}_{V_t^{-1}}+\norm{P_t}_2\theta_{\max}\lambda_{\min}(V_t)^{-1/2}.
\]
By the property of $V_t$ (Lemma~\ref{lem:Gram}) $\lambda_{\min}(V_t) \ge \max\{x_{\max}^2,1\}h_t \ge 1$.
It follows that
\[
\norm{\widehat{\theta}_t - \theta_{\star}}_{V_t}  \le \big(\norm{P_t}_2+1\big)\norm{S_t}_{V_t^{-1}}+\norm{P_t}_2\theta_{\max}.
\]
Thus,
\[
\begin{split}
&\PP\left(\bigcap_{t : |\Acal_t| \ge h_t} \left[\left\{ \norm{\widehat{\theta}_t - \theta_{\star}}_{V_t} > x\right\} \cap \Scal_t\right]\right)\\
& \le \PP\left(\bigcap_{t : |\Acal_t| \ge h_t} \left[\left\{ \big(\norm{P_t}_2+1\big)\norm{S_t}_{V_t^{-1}}+\norm{P_t}_2\theta_{\max} > x \right\} \cap \Scal_t \right]\right) + \delta.
\end{split}
\]
By Lemma~\ref{lem:P_t_bound}, we have $\PP(\|P_t\|_2>5) \le \delta/t^2$ for all $t$ such that $|\Acal_t| \ge h_t$.
Thus,
\begin{align*}
&\PP\left(\bigcap_{t: |\Acal_t| \ge h_t} \left[\left\{ \big(\norm{P_t}_2+1\big)\norm{S_t}_{V_t^{-1}}+\norm{P_t}_2\theta_{\max} > x \right\} \cap \Scal_t \right]\right) + \delta \\
&\le\PP\bigg(\bigcap_{t:|\Acal_t| \ge h_t} \left[\left\{ 6\norm{S_t}_{V_t^{-1}}+5\theta_{\max} > x \right\} \cap \Scal_t \right] \bigg) +\PP\Big(\cup_{t:|\Acal_t| \ge h_t}\big\{\|P_t\|>5\big\}\Big)+\delta\\
&\le\PP\bigg(\bigcap_{t:|\Acal_t| \ge h_t} \left[\left\{ \norm{S_t}_{V_t^{-1}} > \frac{x}{6} - \frac{5}{6}\theta_{\max} \right\} \cap \Scal_t \right] \bigg) +2\delta.
\end{align*}
By Lemma~\ref{lem:S_t_bound}, setting $x=5\theta_{\max}+\frac{6\sigma}{\gamma}\sqrt{d\log\frac{1+t}{\delta}}$ completes the proof of the theorem.
\end{proof}

\subsection{Probability Bounds for the Norms}
This section provides the probability inequalities for the norms of the core terms: \( P_t \) and \( S_t \). 
These bounds are key to proving the convergence of the estimator. 
We will use tail inequalities for random matrices and sums to obtain the desired results.

\begin{lemma}
\label{lem:P_t_bound}
Suppose \( t\ge T_1 \). 
Then, with probability at least \( 1 - \frac{\delta}{t^2} \), the spectral norm \( \|P_t\|_2 \le 5 \).
\end{lemma}

\begin{proof}
By the definition of \( P_t \), we have
\[
\|P_t\|_2 = \|V_t^{-1/2}(V_t - A_t)(A_t +  I_d)^{-1} V_t^{1/2}\|_2.
\]
Expanding this,
\[
\|P_t\|_2 = \|V_t^{1/2}(A_t +  I_d)^{-1} V_t^{1/2} - V_t^{-1/2} A_t (A_t +  I_d)^{-1} V_t^{1/2}\|_2.
\]
Next, we simplify further:
\[
\|P_t\|_2 = \|V_t^{1/2}(A_t +  I_d)^{-1} V_t^{1/2} - I_d +  V_t^{-1/2}(A_t + I_d)^{-1} V_t^{1/2}\|_2.
\]
This gives us the bound:
\[
\|P_t\|_2 \le \|V_t^{1/2}(A_t + I_d)^{-1} V_t^{1/2} - I_d\|_2 +  \|V_t^{-1/2}(A_t +  I_d)^{-1} V_t^{1/2}\|_2.
\]
Since \( V_t \) and \( A_t \) are real symmetric and positive semidefinite, we have the following:
\[
\|V_t^{-1/2}(A_t +  I_d)^{-1} V_t^{1/2}\|_2 \le \frac{1}{\lambda_{\min}(V_t)} \|V_t^{1/2}(A_t +  I_d)^{-1} V_t^{1/2}\|_2.
\]
Thus, we obtain:
\[
\|P_t\|_2 \le \|V_t^{1/2}(A_t +  I_d)^{-1} V_t^{1/2} - I_d\|_2 + \frac{1}{\lambda_{\min}(V_t)} \|V_t^{1/2}(A_t + I_d)^{-1} V_t^{1/2}\|_2.
\]
Let $F_t:=\sum_{s\in[t]\setminus\Acal_t} \sum_{i=1}^{N_s} Z_{i,s}Z_{i,s}^\top$ denote the Gram matrix for the rounds without orthogonal basis augmentation.
Then by the property of the Gram matrix $V_t$ (Lemma~\ref{lem:Gram}),
\[
V_t \prec \underset{F_t}{\underbrace{\sum_{s\in[t]\setminus\Acal_t} \sum_{k=1}^{K} X_{k,s}X_{k,s}^\top + 2h_t \max\{x_{\max}^2,1\}I_d}}.
\]
Note that $F_t=\sum_{s\in[t]\setminus\Acal_t} \sum_{k=1}^{K} X_{k,s}X_{k,s}^\top$ which does not depend on both $a_t$ and $\tilde{a}_t$.
Thus the matrix,
\[
V_t^{-1/2}(A_t +  I_d) V_t^{-1/2} \succeq \sum_{s=1}^{t} \sum_{i=1}^{N_s} \frac{\II(\tilde{a}_s=i)}{\phi_{i,s}} F_{t}^{-1/2} Z_{i,s}Z_{i,s}^\top F_{t}^{-1/2}
\]
Let $\Fcal_t$ denote the sigma-algebra generated by $\big\{(X_{k,s},a_s,\tilde{a}_s):k\in[K],s\in[t]\big\}$ and $\big\{(a_s,\tilde{a}_s):s\in[t]\big\}\cup\{a_{t+1}\}$. 
Because the distribution of the hypothetical action $\tilde{a}_t$ is $\tilde{\pi}_t$ given $a_t$ and $F_t$ depends only on the original contexts $\{X_{k,s}:k\in[K],s\in[t]\}$, 
\begin{align*}
&\EE\Big[\sum_{i=1}^{N_s} \frac{\II(\tilde{a}_s=i)}{\phi_{i,s}} F_{t}^{-1/2} Z_{i,s}Z_{i,s}^\top F_{t}^{-1/2} \Big| \Fcal_{t-1}\Big] \\
&= F_{t}^{-1/2}\CE{\sum_{i=1}^{N_s} \frac{\II(\tilde{a}_s=i)}{\phi_{i,s}}  }{\Fcal_{t-1}}Z_{i,s}Z_{i,s}^\top F_{t}^{-1/2}\\
&= \sum_{i=1}^{N_s}F_{t}^{-1/2}Z_{i,s}Z_{i,s}^\top F_{t}^{-1/2},
\end{align*}
where the first equality holds because hypothetical contexts $\{Z_{i,s}:i\in[N_s]\}$ is deterministic given the original contexts $\{X_{k,t}:k\in[K]\}$ and the action $a_t$.
Observe that for each $s\in[t]$,
\begin{align*}
&\lambda_{\max} \Big( \sum_{i=1}^{N_s} \frac{\II(\tilde{a}_s=i)}{\phi_{i,s}} F_{t}^{-1/2} Z_{i,s}Z_{i,s}^\top F_{t}^{-1/2}\Big) \\
& \le \frac{d}{1-\gamma} \max_{i\in[N_s]}\lambda_{\max}(F_{t}^{-1/2} Z_{i,s}Z_{i,s}^\top F_{t}^{-1/2})  \\
& = \frac{d}{1-\gamma} \max_{i\in[N_s]} \|Z_{i,s}\|_{F_{t}^{-1/2}}^2.
\end{align*}
By definition of the hypothetical contexts \eqref{eq:new_contexts}, for any $s\in[t]$ and $i\in[N_s]$,
\begin{align*}
F_{t}
& \succeq 2\max\{x_{\max}^2,1\} h_t I_d \\
& \succeq Z_{i,s}Z_{i,s}^\top + \max\{x_{\max}^2 ,1\} h_t I_d,
\end{align*}
where the last inequality uses $\|Z_{i,s}\|_2 \le \max\{x_{\max},1\}$.
Using Sherman-Morrison formula, for any $s\in[t]$ and $i\in[N_s]$,
\begin{align*}
\|Z_{i,s}\|_{F_{t}^{-1}}^2 &\le  Z_{i,s}^\top \big(Z_{i,s}Z_{i,s}^\top + \max\{x_{\max}^2 ,1\} h_t I_d\big)^{-1}Z_{i,s} \\
& \le \frac{\|Z_{i,s}\|_2^2}{\max\{x_{\max}^2,1\}h_t + \|Z_{i,s}\|_2^2} \\
& \le \frac{1}{h_t+1},
\end{align*}
where the last inequality uses $\|Z_{i,s}\|_2\le\max\{x_{\max},1\}$.
By definition of $h_t$, we obtain,
\[
\lambda_{\max} \Big( \sum_{i=1}^{N_s} \frac{\II(\tilde{a}_s=i)}{\phi_{i,s}} F_t^{-1/2} Z_{i,s}Z_{i,s}^\top F_t^{-1/2}\Big)  \le  \frac{d}{(1-\gamma)h_t} \le \frac{1/2-e^{-1}}{2}\big(\log \frac{d t^2}{\delta}\big)^{-1}
\]
Then by Lemma~\ref{lem:matrix_neg}, with probability at least $1-\delta/t^2$
\begin{align*}
&2(1/2-e^{-1})^{-1}\log\frac{dt^2}{\delta}\sum_{s=1}^{t} \sum_{i=1}^{N_s} \frac{\II(\tilde{a}_s=i)}{\phi_{i,s}} F_t^{-1/2} Z_{i,s}Z_{i,s}^\top F_t^{-1/2}\\
&\succeq 2(1/2-e^{-1})^{-1}\log\frac{dt^2}{\delta} (1-e^{-1}) F_t^{-1/2} V_t F_t^{-1/2}- \log \frac{d t^2}{\delta} I_d \\
& \succeq (1/2-e^{-1})^{-1}\log\frac{dt^2}{\delta} (1-e^{-1}) I_d- \log \frac{d t^2}{\delta} I_d,
\end{align*}
where the last inequality uses the fact that $V_t \succeq F_t/2$, which can be derived by Lemma~\ref{lem:Gram}.
Thus,
\[
\sum_{s=1}^{t} \sum_{i=1}^{N_s} \frac{\II(\tilde{a}_s=i)}{\phi_{i,s}} F_t^{-1/2} Z_{i,s}Z_{i,s}^\top F_t^{-1/2} \succeq \frac{1-e^{-1}}{2}-\frac{1/2-e^{-1}}{2} =\frac{1}{4},
\]
which implies, 
\begin{align*}
V_t^{-1/2}(A_t + I_d)V_t^{-1/2} \succeq \frac{1}{4}I_d.
\end{align*}
Thus,
\[
\|V_t^{-1/2}(A_t +  I_d) V_t^{-1/2} \|_2 \le 4,
\]
which implies
\[
\|P_t\|_2 \le 4 - 1 + \frac{4}{\lambda_{\min}(V_t)}.
\]
Since ${\min}(V_t) \ge h_t \max\{x_{\max}^2,1\} \ge 4$, we have
\[
\|P_t\|_2 \le 5
\]
\end{proof}

\begin{lemma}
\label{lem:S_t_bound}
For any \( \delta \in (0, 1) \),
\[
\mathbb{P}\left( \bigcap_{t : |\Acal_t| \ge h_t} \left[ \left\{ \|S_t\|_{V_t^{-1}} > \frac{\sigma}{\gamma} \sqrt{d \log  \frac{1 + t}{\delta} } \right\} \cap \Scal_t \right] \right) \le \delta.
\]
\end{lemma}

\begin{proof}
Under the event \( \Scal_t \), we have:
\[
\|S_t\|_{V_t^{-1}} := \left\| \sum_{s=1}^{t} \sum_{i=1}^{r_s + 1} \frac{\mathbb{I}(\tilde{a}_s = i)}{\phi_{i,s}} \left(W_{i,s} - Z_{i,s}^{\top} \theta_\star\right) Z_{i,s} \right\|_{V_t^{-1}}.
\]
Simplifying this, we get:
\[
\|S_t\|_{V_t^{-1}} = \frac{1}{\gamma} \left\| \sum_{s=1}^{t} \eta_{a_s, s} X_{a_s, s} \right\|_{V_t^{-1}}.
\]
Next, we observe that the Gram matrix \( V_t \) satisfies:
\[
V_t \succeq \sum_{s=1}^{t} X_{a_s, s} X_{a_s, s}^{\top} + \sum_{s \in \Acal_t} \sum_{i=1}^{d} Z_{i,s} Z_{i,s}^{\top}.
\]
By the definition of the new contexts, we have:
\[
\sum_{s \in \Acal_t} \sum_{i=1}^{d} Z_{i,s} Z_{i,s}^{\top} \succeq h_t \max\{x_{\max}^2, 1\} I_d.
\]
Thus,
\[
V_t \succeq \sum_{s=1}^{t} X_{a_s, s} X_{a_s, s}^{\top} + \max\{x_{\max}^2,1\} I_d.
\]
Using Lemma 9 and Lemma 10 from \citet{abbasi2011improved}, with probability at least \( 1 - \delta \), we get:
\[
\frac{1}{\gamma} \left\| \sum_{s=1}^{t} \eta_{a_s, s} X_{a_s, s} \right\|_{V_t^{-1}} \le \frac{\sigma}{\gamma} \sqrt{d \log \left( \frac{1 + t }{\delta} \right)}.
\]
This completes the proof.
\end{proof}

\subsection{Proof of Lemma~\ref{lem:super_unsaturated_arms}}
\label{sec:low_regret_proof}
\begin{proof}
For \(k \in [K]\) and \(t \in [T]\), let \(\tilde{Y}_{k,t} := X_{k,t}^\top \tilde{\theta}_{k,t}\) denote the estimated reward for arm \(k\). Define the maximizer \(\tilde{M}_{t} := \arg\max_{k \in [K]} \tilde{Y}_{k,t}\). Since \texttt{HCSA+TS} selects the arm that maximizes \(\tilde{Y}_{k,t}\), the distribution of \(a_t\) matches that of \(\tilde{M}_t\), i.e.,
\[
\CP{a_t \in \Pcal_t}{\Hcal_t} = \CP{\tilde{M}_t \in \Pcal_t}{\Hcal_t}.
\]
Suppose the estimated reward for the optimal arm, \(\tilde{Y}_{a_t^{\star},t}\), exceeds \(\tilde{Y}_{j,t}\) for all \(j \in [K] \setminus \Pcal_t\). Since the optimal arm \(a_t^{\star}\) is always in \(\Pcal_t\) by definition, it follows that \(a_t^{\star} \in \Pcal_t\). Thus:
\[
\CP{\tilde{M}_t \in \Pcal_t}{\Hcal_t} \ge \CP{\bigcap_{j \in [K] \setminus \Pcal_t} \{\tilde{Y}_{j,t} < \tilde{Y}_{a_t^{\star},t}\}}{\Hcal_t}.
\]
Let \(Z_{j,t} := \tilde{Y}_{a_t^{\star},t} - \tilde{Y}_{j,t} - (X_{a_t^{\star},t} - X_{j,t})^{\top} \Estimator{t-1}\). Then:
\[
\CP{\tilde{M}_t \in \Pcal_t}{\Hcal_t} \ge \CP{\bigcap_{j \in [K] \setminus \Pcal_t} \{Z_{j,t} > (X_{j,t} - X_{a_t^{\star},t})^{\top} \Estimator{t-1}\}}{\Hcal_t}.
\]
Given \(\Hcal_t\), \(\{Z_{j,t} : j \in [K] \setminus \Pcal_t\}\) are Gaussian random variables with mean 0 and variance \(v_t^2 (\|X_{a_t^{\star},t}\|_{V_{t-1}^{-1}}^2 + \|X_{j,t}\|_{V_{t-1}^{-1}}^2)\). For each \(j \notin \Pcal_t\), we have:
\[
(X_{j,t} - X_{a_t^{\star},t})^{\top} \Estimator{t-1} \le 2x_t \norm{\Estimator{t-1} - \theta_{\star}}_{V_{t-1}} - \Diff{j}{t} \le -\sqrt{\norm{X_{a_t^{\star},t}}_{V_{t-1}^{-1}}^2 + \norm{X_{j,t}}_{V_{t-1}^{-1}}^2}.
\]
Thus:
\[
\CP{\tilde{M}_t \in \Pcal_t}{\Hcal_t} \ge \CP{\frac{Z_{j,t}}{v_t \sqrt{\norm{X_{a_t^{\star},t}}_{V_{t-1}^{-1}}^2 + \norm{X_{j,t}}_{V_{t-1}^{-1}}^2}} > -\frac{1}{v_t}, \forall j \notin \Pcal_t}{\Hcal_t}.
\]
Since \(Z_{j,t}\) is Gaussian with variance normalized to 1,
\[
\begin{split}
\CP{\tilde{M}_{t} \in \Pcal_{t}}{\Hcal_t} \ge & \CP{\frac{Z_{j,t}}{v_{t}\sqrt{\norm{X_{a_t^{\star},t}}_{V_{t-1}^{-1}}^{2} + \norm{X_{j,t}}_{V_{t-1}^{-1}}^{2}}} > -\frac{1}{v_t}, \forall j \notin \Pcal_{t}}{\Hcal_t} \\
:= & \CP{Y_j > -v_t^{-1}, \forall j \neq \Pcal_{t}}{\Hcal_t},
\end{split}
\]
where 
\[
Y_j := \frac{Z_{j,t}}{v_t \sqrt{\norm{X_{a_t^{\star},t}}_{V_{t-1}^{-1}}^{2} + \norm{X_{j,t}}_{V_{t-1}^{-1}}^{2}}}
\]
is a standard Gaussian random variable given \(\Hcal_t\). Therefore, we have:
\[
\CP{Y_j \le -v_t^{-1}}{\Hcal_t} \le \exp\left(-\frac{1}{2v_t^2}\right).
\]
Now, setting \( v_t = \left( 2 \log \frac{K (t+1)^2}{\delta} \right)^{-1/2} \), we get:
\[
\CP{Y_j \le -v_t^{-1}}{\Hcal_t} \le \exp\left(-\log \frac{(t+1)^2}{\delta}\right) = \frac{\delta}{K (t+1)^2}.
\]
Thus, we obtain:
\[
\begin{aligned}
\CP{\tilde{M}_t \in \Pcal_{t}}{\Hcal_t} \ge & \ 1 - \CP{Y_j \le -v_t^{-1}, \exists j \notin \Pcal_{t}}{\Hcal_t} \\
\ge & \ 1 - \sum_{j \notin \Pcal_{t}} \CP{Y_j \le -v_t^{-1}}{\Hcal_t} \\
\ge & \ 1 - \frac{\delta}{(t+1)^2}.
\end{aligned}
\]
This completes the proof.
\end{proof}

\subsection{Proof of Lemma~\ref{lem:elliptical}}
\label{sec:proof_elliptical}
\begin{proof}
By definition of $\Acal_t$, for $t$ such that $t\le h_t$, we have \(t \in \Acal_t\) and $1\in\Acal_t$ for all $t\in[T]$.
By definition of the new contexts~\eqref{eq:new_contexts}, for any $t\in[T]$,
\begin{align*}
V_t =\sum_{s=1}^{t} \sum_{i=1}^{N_s} Z_{i,t} Z_{i,t}^\top &\succeq \sum_{s\in[t]\setminus\Acal_t}\sum_{i=1}^{N_s} Z_{i,s}Z_{i,s}^\top + \sum_{i=1}^{d} u_{i,1} u_{i,1}^\top \\
&\succeq  \sum_{s\in[t]\setminus\Acal_t}\sum_{i=1}^{N_s} Z_{i,s}Z_{i,s}^\top+\max\{x_{\max}^2, 1\} I_d \\
&= \sum_{s\in[t]\setminus\Acal_t}\sum_{k=1}^{K} X_{k,s}X_{k,s}^\top+\max\{x_{\max}^2, 1\} I_d
\end{align*}
It follows that \[
V_t \succeq \underset{W_t}{\underbrace{\sum_{s\in[t]\setminus\Acal_t} X_{\nu_s,s} X_{\nu_s,s}^\top+\max\{x_{\max}^2, 1\} I_d}}
\] 
and \(x_t \leq \|X_{\nu_t,t}\|_{W_{t-1}^{-1}}\). 
Applying Lemma 11 from \citet{abbasi2011improved}:
\[
\sum_{t\in[T]\setminus\Acal_T} x_t^2 \leq \sum_{t\in[T]\setminus\Acal_T} \|X_{\nu_t,t}\|_{W_{t-1}^{-1}}^2 \leq 2\log\frac{\det(W_T)}{\det(\max\{x_{\max}^2, 1\}I_d)}.
\]
By AM-GM inequality, \(\det(W_T) \leq \left(\frac{\Trace{W_T}}{d}\right)^d \leq \left(\frac{T \max\{x_{\max}^2, 1\}}{d}\right)^d\) holds. 
Thus,
\[
\sum_{t\in[T]\setminus\Acal_T} x_t^2 \leq 2d \log\frac{\Trace{W_T}}{\max\{x_{\max}^2, 1\}d} \leq 2d \log\frac{T\max\{x_{\max}^2,1\}}{\max\{x_{\max}^2, 1\}d} \leq 2d \log\frac{T}{d},
\]
which completes the proof.
\end{proof}

\subsection{A Matrix Concentration Inequality}

\begin{lemma}[Matrix concentration inequality]
\label{lem:matrix_neg}
Let \(\{M_{s}: s \in [t]\}\) be a \(\RR^{d \times d}\)-valued stochastic process adapted to the filtration \(\{\Fcal_{s}: s \in [t]\}\). Suppose \(M_s\) are nonnegative definite symmetric matrices such that \(\Maxeigen{M_s} \le 1\). Then, with probability at least \(1 - \delta\),
\[
\sum_{s=1}^{t} M_s \succeq (1 - e^{-1}) \sum_{s=1}^{t} \CE{M_s}{\Fcal_{s-1}} - \log\frac{d}{\delta} I_d.
\]
In addition, with probability at least \(1 - \delta\),
\[
\sum_{s=1}^{t} M_s \preceq \left(e - 1\right) \sum_{s=1}^{t} \CE{M_s}{\Fcal_{s-1}} + \log\frac{d}{\delta} I_d.
\]
\end{lemma}

\begin{proof}
This proof is an adapted version of the argument from \citet{tropp2012user}. 

For the lower bound, it is sufficient to prove that
\[
\Maxeigen{-\sum_{s=1}^{t} M_s + (1 - e^{-1}) \sum_{s=1}^{t} \CE{M_s}{\Fcal_{s-1}}} \le \log \frac{d}{\delta},
\]
with probability at least \(1 - \delta\). By the spectral mapping theorem,
\begin{align*}
&\exp\left( \Maxeigen{-\sum_{s=1}^{t} M_s + (1 - e^{-1}) \sum_{s=1}^{t} \CE{M_s}{\Fcal_{s-1}}} \right) \\
&\le \Maxeigen{\exp\left( -\sum_{s=1}^{t} M_s + (1 - e^{-1}) \sum_{s=1}^{t} \CE{M_s}{\Fcal_{s-1}} \right)} \\
&\le \Trace{\exp\left( -\sum_{s=1}^{t} M_s + (1 - e^{-1}) \sum_{s=1}^{t} \CE{M_s}{\Fcal_{s-1}} \right)}.
\end{align*}
Taking the expectation of both sides gives:
\begin{align*}
&\EE \exp\left( \Maxeigen{-\sum_{s=1}^{t} M_s + (1 - e^{-1}) \sum_{s=1}^{t} \CE{M_s}{\Fcal_{s-1}}} \right) \\
&\le \EE \Trace{\exp\left( -\sum_{s=1}^{t} M_s + (1 - e^{-1}) \sum_{s=1}^{t} \CE{M_s}{\Fcal_{s-1}} \right)} \\
&= \EE \Trace{\CE{\exp\left( -\sum_{s=1}^{t-1} M_s + (1 - e^{-1}) \sum_{s=1}^{t} \CE{M_s}{\Fcal_{s-1}} + \log \exp\left( -M_t \right) \right)}{\Fcal_{t-1}}} \\
&\le \EE \Trace{\exp\left( -\sum_{s=1}^{t-1} M_s + (1 - e^{-1}) \sum_{s=1}^{t} \CE{M_s}{\Fcal_{s-1}} + \log \CE{\exp\left( -M_t \right)}{\Fcal_{t-1}} \right)}.
\end{align*}
The last inequality follows from Lieb's theorem \citep{tropp2015introduction}. 
Define a function \( f_{\lambda}: [0, 1] \to \RR \) as \( f_{\lambda}(x) = e^{\lambda x} - x (e^{\lambda} - 1) - 1 \). Then \( f_{\lambda}(x) \) is convex with \( f_{\lambda}(0) = f_{\lambda}(1) = 0 \) for all \(\lambda \in \RR\). Thus, \( e^{-x} \le 1 + x(e^{-1} - 1) \) for \(x \in [0, 1]\).
Because the eigenvalues of \( M_s \) lie in \([0, 1]\), by the spectral mapping theorem, we have
\[
\CE{\exp\left(-M_t\right)}{\Fcal_{t-1}} \preceq I + (e^{-1} - 1) \CE{M_t}{\Fcal_{t-1}} \preceq \exp\left( -(1 - e^{-1}) \CE{M_t}{\Fcal_{t-1}} \right).
\]
Thus, we obtain:
\begin{align*}
& \EE \exp\left( \Maxeigen{-\sum_{s=1}^{t} M_s + (1 - e^{-1}) \sum_{s=1}^{t} \CE{M_s}{\Fcal_{s-1}}} \right) \\
&\le \EE \Trace{\exp\left( -\sum_{s=1}^{t-1} M_s + (1 - e^{-1}) \sum_{s=1}^{t} \CE{M_s}{\Fcal_{s-1}} + \log \exp\left( -M_t \right) \right)} \\
&= \EE \Trace{\exp\Bigl( -\sum_{s=1}^{t-1} M_s + (1 - e^{-1}) \sum_{s=1}^{t} \CE{M_s}{\Fcal_{s-1}} + \log \exp( -(1 - e^{-1}) \CE{M_t}{\Fcal_{s-1}}) \Bigr)} \\
&= \EE \Trace{\exp\left( -\sum_{\tau=1}^{t-1} M_s + (1 - e^{-1}) \sum_{s=1}^{t} \CE{M_s}{\Fcal_{s-1}} \right)} \\
&\le \vdots \\
&\le \EE \Trace{\exp\left(O\right)} = d.
\end{align*}

Now, by Markov's inequality:
\begin{align*}
&\PP\left( \Maxeigen{-\sum_{s=1}^{t} M_s + (1 - e^{-1}) \sum_{s=1}^{t} \CE{M_s}{\Fcal_{s-1}}} > \log \frac{d}{\delta} \right) \\
&\le \EE \exp\left( \Maxeigen{-\sum_{s=1}^{t} M_s + (1 - e^{-1}) \sum_{s=1}^{t} \CE{M_s}{\Fcal_{s-1}}} \right) \frac{\delta}{d} \\
&\le \delta.
\end{align*}
For the upper bound, we can prove:
\[
\Maxeigen{\sum_{s=1}^{t} M_s - \left( e - 1 \right) \sum_{s=1}^{t} \CE{M_s}{\Fcal_{s-1}}} \le \log \frac{d}{\delta},
\]
in a similar manner, using the fact that \( e^{x} \le 1 + (e - 1) x \) for \( x \in [0, 1] \).
\end{proof}

\vskip 0.2in
\bibliography{ref}

\end{document}